\documentclass[11pt,a4paper]{article}
\usepackage[utf8]{inputenc}
\usepackage[T1]{fontenc}
\usepackage{amsmath, amssymb, amsthm}
\usepackage{hyperref}
\usepackage{booktabs}
\usepackage{graphicx}
\usepackage{placeins}
\usepackage[a4paper,hmargin=2.5cm,vmargin=2.5cm,verbose]{geometry}

\newtheorem{theorem}{Theorem}
\newtheorem{definition}{Definition}
\newtheorem{proposition}{Proposition}
\newtheorem{lemma}{Lemma}
\newtheorem{corollary}{Corollary}
\newtheorem{remark}{Remark}
\newtheorem{assumption}{Assumption}

\begin{document}

\title{A Theory of Conditional Collapse under \\ Low-Rank Weight-Space Ablations: \\ I. The Single-Block Theory and Synthetic Validation}
\author{Abdallah Khemais \\ \textit{ISITCOM, University of Sousse}}
\date{July 2026}

\maketitle

\begin{abstract}
Activation patching and weight-space ablation are both used to argue that a
component of a network is causally responsible for a behavior, yet they act
on different objects: one forward pass, versus the parameters behind every
forward pass. We ask when they agree.

We study an idealized model in which a conditional computation is carried
additively through a residual stream, $F(x)=F_0(x)+\sum_i\alpha_i(x)v_i$, and
read out by a linear functional, and prove three exact results. First,
deleting a subset of the carriers $v_i$ collapses a matched input pair onto
one and the same unconditional output \emph{if and only if} the removal is
symmetric on the pair and leaves no contrast outside it; the resulting error
is deterministic, with polarity the sign of the pair's mean margin. Since
those two conditions are exact equalities, we also give the exact error
identity when they hold only approximately, so the criterion degrades
gracefully rather than describing a null set. Second,
patching a carrier moves the readout by that carrier's donor--receiver
\emph{contrast}, whereas ablating it moves the readout by its \emph{absolute
level} at the receiver. Neither bounds the other, and we construct matched
pairs on which every single-carrier patch flips the decision while no
single-carrier ablation does: the redundancy regime in which patching
overstates importance and ablation understates it. Third, for an attention
head composed with its own layer's normalization and MLP, we derive an exact
first-order formula, with a provably second-order remainder, for the
interaction term the idealized model sets to zero, and show it vanishes
identically whenever the MLP alone is ablated but not, in general, when a
head is; we connect it to the idealized selector's own error against the
true network.

Small transformers trained on a synthetic conditional task illustrate all
three predictions. Over thirty-nine ablation configurations the measured
interaction is strongly rank-correlated with how well the idealized model
predicts the edited network's behavior (Spearman $-0.83$); a clean
separation we first read off fourteen of those configurations does
\emph{not} survive on the remaining twenty-five, and we report the weaker
monotone claim that does. A second task and architecture (an inverse,
value-to-key recall the first task never requires, on a smaller network)
reproduces the same monotone relationship, the same patch-equals-weight-edit
exactness, and a further instance of the polarity reversal.

The single-block interaction result derived here extends past one residual
block, and the synthetic validation reported here is tested against a real
pretrained model, in a companion analysis that takes the present theory
further along both axes.
\end{abstract}

\section{Introduction}
\label{sec:intro}

Mechanistic interpretability certifies that a component of a network is
causally responsible for a behavior by intervening on it, and two families of
intervention dominate. \emph{Activation patching} replaces a component's
activation, inside a single forward pass, by the value it takes on a second
input, and asks whether the model's decision follows
\cite{vig2020,meng2022,wang2023}. \emph{Weight-space ablation} instead edits
the parameters (zeroing a component, or projecting a low-rank direction out
of a weight matrix \cite{arditi2024}) and asks the same question of the
edited network. Both are read as evidence for the same informal claim, that a
given component carries a given behavior, and both are frequently interpreted
through a common lens of causal abstraction \cite{geiger2021}.

The two interventions are nevertheless different operators on different
objects: patching alters one realized computation, ablation alters the
parameters that generate every computation. Nothing in current practice
specifies when they should agree, and there is direct evidence that they need
not. The \emph{Hydra effect} \cite{mcgrath2023} documents that ablating an
attention layer in a language model causes downstream layers to change their
behavior and compensate; an ablation-based importance score is therefore
measured on a network that has already partly repaired the damage, and can
fall far below the importance the same component is assigned by an
intervention that does not give the rest of the network that opportunity.
Redundant, superposed codes \cite{elhage2022} make the same mismatch
structurally likely rather than accidental.

This paper makes the mismatch exact. We work in an idealized model, the
\emph{abstract conditional model} of Section~\ref{sec:model}, in which a
conditional computation is carried additively through a residual stream,
\[
F(x) \;=\; F_0(x) + \sum_{i=1}^{k}\alpha_i(x)\,v_i,
\]
with $F_0$ the unconditional part, $v_i$ fixed directions, $\alpha_i$ scalar
coefficients (we call the pair $(v_i,\alpha_i)$ a \emph{carrier}), and read
out by a linear functional, $s(x)=\psi(F(x))+b$, whose sign is the decision.
Fix a \emph{matched pair} $(x_A,x_B)$, two inputs that should receive the
same answer through opposite branches of the conditional, with margins
$g_A=s(x_A)>0>g_B=s(x_B)$, and write $\beta_i:=\psi(v_i)$. Our results are
then the following, in the order in which they are proved.

\paragraph{Conditional collapse (Section~\ref{sec:collapse}).} Deleting the
carriers indexed by a subset $S$ maps \emph{both} members of the pair to one
and the same value $\bar s=\tfrac12(g_A+g_B)$ if and only if two conditions
hold: the removed mass is symmetric on the pair ($\bar q_S=0$), and no
contrast survives outside $S$ ($\Psi_S=0$). When they hold, exactly one input
is misclassified, its error is deterministic rather than noisy, the surviving
branch is $\operatorname{sign}\bar s$, and any two subsets satisfying the
criterion produce the identical collapse.

\paragraph{Patching--ablation dissociation (Section~\ref{sec:dissociation}).}
Patching carrier $i$ from donor to receiver moves the readout by exactly
$\beta_i\delta_i$ with $\delta_i:=\alpha_i(x_A)-\alpha_i(x_B)$ (the
carrier's \emph{contrast}), whereas ablating carrier $i$ moves it by exactly
$-\beta_i\alpha_i(x_B)$ (the carrier's \emph{absolute level} at the
receiver). Neither quantity bounds the other. We construct matched pairs with
$n\ge2$ carriers on which every single-carrier patch flips the decision while
no single-carrier ablation does, which is precisely the regime that produces
patch recoveries exceeding $1$ (overshoot) alongside ablations that appear to
show the component is dispensable: the shape of the self-repair
observations above, obtained here with no repair mechanism at all.

\paragraph{Nonlinear interaction (Section~\ref{sec:interaction}).} The
idealized model treats carriers as independent. For the one composition where
that is architecturally false (an attention head and its own layer's MLP,
separated by an RMSNorm that reads the head's output), we compute the error
exactly:
\[
\Delta(x) \;=\; (I-Q)\Bigl[-Dg\bigl(r_1(x)\bigr)\eta(x) - R(x)\Bigr],
\qquad \|R(x)\|\le\Lambda\|\eta(x)\|^2 ,
\]
where $\eta(x)$ is the head's ablated write-in, $Q$ the MLP's ablation
projector, $g$ the normalization--MLP composition, and the displayed
remainder bound holds under the bounded-curvature hypothesis stated in
Theorem~\ref{thm:interaction}, a hypothesis whose constant $\Lambda$ can in
fact be exhibited in closed form, from the trained weights alone, as a
companion analysis to this one shows, so that it need not be assumed
unverified. The remainder is provably second order, and $\Delta$ vanishes
identically whenever the MLP alone is ablated ($\eta(x)\equiv0$), but not in
general otherwise: ablating the head alone perturbs its own block's MLP
output regardless of whether the MLP is separately touched, and we show
precisely what this costs the first two theorems' predictions when applied
to the true network rather than the idealized one.

\medskip
Section~\ref{sec:experiments} illustrates all three predictions on small
transformers trained on a synthetic conditional task. That section is
deliberately minimal: its purpose is to show that the predictions are
realized in a trained network rather than vacuous under the stated
assumptions, not to conduct a validation campaign. It also reports a negative
result we consider load-bearing. An apparent clean separation, on which the
interaction magnitude split fourteen ablation configurations into those the
idealized model describes and those it does not with an empty band between,
does not survive when the same frozen criterion is applied to twenty-five
further configurations from the same networks; what survives is a strong
monotone association rather than a threshold, and we state the claim at that
strength. A second finding is more structural than surprising, and we report
it for the same reason: the robust collapse criterion of
Section~\ref{sec:robust} is a \emph{sufficient} tolerance condition, and on
these networks it is met by only one matched pair in ten, and by no single
configuration on more than $71.7\%$ of its pairs, because the condition
compares two quantities of unequal, architecturally fixed scale, not because
the underlying collapse fails to occur, which it does on four times as many
pairs as the criterion certifies. To check that neither the theory's
predictions nor this section's own negative results are an artifact of the
marker task specifically, Section~\ref{sec:second-task} repeats the checks
above on a second, freshly trained instance with a different conditional
mechanism (inverse recall rather than a fixed relabeling) and a smaller
architecture; the predictions replicate, including a further instance of the
polarity reversal, though the carrier count needed widening past this
section's own threshold to give the interaction/fidelity check enough points
to be informative, itself reported rather than smoothed over.
The single-block result of Theorem~\ref{thm:interaction} itself extends past
one block, to an exact identity for the interaction between two arbitrarily
distant layers, and is tested, together with the collapse and dissociation
predictions above, on an emergent circuit found, not designed, in a real
pretrained model; we develop both extensions in a companion analysis rather
than here, so as to keep the present paper to the single-block theory and its
synthetic-task validation.
Section~\ref{sec:discussion} separates what is established from what the
model leaves open.

\section{Related Work}
\label{sec:related}

\paragraph{Causal interventions as an interpretability method.} Patching an
activation from one run into another, and reading the change in the output,
originates in causal mediation analysis applied to language models
\cite{vig2020} and was scaled into causal tracing for factual recall
\cite{meng2022} and into full circuit reconstructions
\cite{wang2023}. Refinements restrict the intervention to individual paths
through the computation graph \cite{goldowsky2023}, automate the search over
candidate edges \cite{conmy2023}, and examine how sensitive the resulting
conclusions are to the choice of corrupted input and evaluation metric
\cite{zhang2024}. Causal abstraction supplies the semantics under which such
an experiment counts as evidence that a network implements a given
high-level algorithm \cite{geiger2021}, including when the high-level
variable is realized in a distributed, rotated subspace rather than in a
single neuron \cite{geiger2024}. The present paper takes this methodology as
given: we do not propose a new intervention or a new search procedure, and
the abstract conditional model of Section~\ref{sec:model} is an
abstraction in exactly the sense of \cite{geiger2021}, specialized to a
binary readout and an additively carried conditional.

\paragraph{Weight-space edits and the additive residual stream.} A separate
line of work intervenes on parameters instead of activations, most directly
by projecting a single low-rank direction out of the weight matrices that
write into the residual stream \cite{arditi2024}. The decomposition
$F=F_0+\sum_i\alpha_i v_i$ we start from is the standard reading of the
residual stream as a sum of component write-ins \cite{elhage2021}, with the
$v_i$ playing the role of the output directions of individual heads or MLPs.
What the present paper adds is not this decomposition but a criterion,
internal to it, for when \emph{removing} a subset of the write-ins forces the
readout to a single value on two inputs that previously separated.

\paragraph{When an intervention misleads.} Several results establish that a
causal intervention can support a conclusion the network does not warrant.
Subspace activation patching can flip a model's behavior through a direction
that the unperturbed model does not use, an interpretability illusion that
survives the usual sufficiency and necessity checks \cite{makelov2024}.
Ablating a component can be compensated by downstream components, so that
measured importance reflects the network's self-repair as much as the
component's role \cite{mcgrath2023}, and redundant or superposed codes
\cite{elhage2022} give a representational reason to expect exactly this.
These are, respectively, an argument that patching can overstate and an
argument that ablation can understate. Our contribution is to place both in
one model and make the gap between them an identity rather than a caution:
Theorem~\ref{thm:dissociation} names the two quantities the two interventions
actually measure, $\beta_i\delta_i$ and $\beta_i\alpha_i(x_B)$, and
Corollary~\ref{cor:sufficiency} exhibits, by explicit construction, matched
pairs on which every single-carrier patch is sufficient while no
single-carrier ablation is necessary.

\paragraph{Concurrent work.} Three papers appearing while this one was being
prepared reach neighbouring conclusions from the activation side.
Vaidyanathan et al.\ \cite{vaidyanathan2026mediators} re-derive the
activation-patching estimand from causal mediation analysis and show that the
natural indirect effect attributed to a component also carries an interaction
term measuring how that component's effect depends on the state of the others;
they prove it scales with the clean-to-patched activation distance, vanishes
when the model is locally affine, and decomposes combinatorially. Their INT is
the activation-space counterpart of the term
Theorem~\ref{thm:interaction} isolates here in weight space, and their
locally-affine condition is the qualitative form of the curvature constant
bounded here in closed form; the two analyses are complementary rather than
competing, since the intervention being analyzed is not the same one.
Gong et al.\ \cite{gong2026coablation} show empirically that dormant backup
components activated by an ablation corrupt first-order importance scores,
which is the mechanism our own understatement argument predicts, measured at
scale on nine models. Guo et al.\ \cite{guo2026sobol} separate ``transports
task-relevant content'' from ``computation degrades when removed'' through
paired interventions; that separation is the empirical shadow of the
dissociation Theorem~\ref{thm:dissociation} states as an identity, though
their ablation is a zeroing in activation space rather than a low-rank edit of
the weights.

\paragraph{What is new here.} Relative to the above, this paper contributes
(i) an \emph{if and only if} criterion for a weight-space ablation to collapse
a conditional onto a single unconditional branch, together with the polarity
of the resulting error (Theorem~\ref{thm:collapse}); (ii) a constructive
separation between patching sufficiency and ablation necessity within a
single model, rather than as two separately observed empirical phenomena
(Corollary~\ref{cor:sufficiency}); and (iii) an exact, second-order-bounded
formula for the interaction that the additive picture omits whenever an
ablated head's own layer's MLP is not itself ablated away entirely
(Theorem~\ref{thm:interaction}), which identifies the single architectural
mechanism through which the first two results cease to be exact, together
with a closed-form bound on the curvature constant controlling its
remainder, established in a companion analysis, so that the second-order
claim is checkable on a given trained network rather than conditional on an
unverified hypothesis.

\section{Problem Formulation}
\label{sec:problem}

We study conditional computations implemented by residual neural networks.
Our objective is not to analyze a particular transformer architecture but to
characterize, in an abstract setting, the effect of removing a low-dimensional
weight-space subspace supporting a learned conditional computation.

Let a residual block act as $x_{l+1}=x_l+F(x_l;W)$, where $W\in\mathbb R^m$
collects the trainable parameters of the components writing into the stream.
By a \emph{low-rank weight-space ablation} we mean an edit that projects a
low-dimensional subspace
\[
\mathcal S \;=\; \operatorname{span}\{u_1,\ldots,u_p\},
\qquad p\ll m,
\]
out of those parameters, with $P_{\mathcal S}$ the orthogonal projector onto
$\mathcal S$: for a weight matrix this reads $W\mapsto(I-P_{\mathcal S})W$ or
$W\mapsto W(I-P_{\mathcal S})$ according to which side of the map the removed
subspace acts on. The concrete instances used below are the head and MLP
edits written out in Remark~\ref{rem:bridge} and
Section~\ref{sec:interaction}, in which $\mathcal S$ is spanned by directions
estimated from data. This is the operator whose effect the paper analyzes,
and it is not the operator activation patching applies: patching alters the
activations realized on one forward pass, leaving $W$ intact, whereas the
edit above changes the map generating every forward pass.

\paragraph{What we analyze.} Working directly with $W'$
is unwieldy and specific to the architecture. We therefore analyze the induced
change in the \emph{residual stream}: Section~\ref{sec:model} posits that the
block's output decomposes additively over the components that write into it,
and Sections~\ref{sec:collapse}--\ref{sec:dissociation} study the deletion of
a subset of those write-ins. The two views are connected concretely, not
in general but for the specific edits we use, in
Remark~\ref{rem:bridge}: ablating an attention head by projecting a direction
out of its output projection subtracts exactly one term of the form
$v\,\alpha(x)$ from the residual stream, which is precisely the deletion
operation Section~\ref{sec:collapse} formalizes. Everything after the present
section is stated in the output space; the weight-space operator above is
what realizes it.

Two features of a \emph{conditional} computation matter for what follows.
It is carried additively, so removing it is a subtraction rather than a
retraining; and it is what makes the block's output differ between two inputs
that the network must nevertheless answer identically. We do not assume it is
gated off outside some region of input space: the coefficients
$\alpha_i(x)$ of Section~\ref{sec:model} are unrestricted scalars, and in the
construction of Corollary~\ref{cor:sufficiency} every carrier is active on
both members of the pair. The question the paper answers is when deleting
such a computation forces the two inputs to the same answer, and how that
differs from what patching the same component would have suggested.

\section{Abstract Conditional Model}
\label{sec:model}

This section introduces the abstract mathematical model underlying the
subsequent theoretical analysis. The objective is to isolate the essential
geometric structure responsible for conditional collapse independently of any
specific neural architecture.

\subsection{Residual Decomposition}

Assume that the residual mapping admits the decomposition
\[
F(x) \;=\; F_0(x) + \sum_{i=1}^{k}\alpha_i(x)\,v_i,
\]
where $F_0$ denotes the unconditional computation, $v_i\in\mathbb R^d$ are
fixed feature directions, and $\alpha_i(x)$ are scalar selector functions. The
vectors $\mathcal V=\{v_1,\ldots,v_k\}$ represent the directions supporting the
conditional computation. We call the pair $(v_i,\alpha_i)$ the $i$-th
\emph{carrier}, a term we also use for its concrete realizations (an
attention head or an MLP writing into the stream) in
Sections~\ref{sec:interaction} and~\ref{sec:experiments}.

\subsection{Low-Rank Support}

We assume $\dim\operatorname{span}(\mathcal V)=r$ with $r\ll d$, so the
conditional computation occupies only a low-dimensional subspace of the
representation space. We denote this support by
$\mathcal C=\operatorname{span}(\mathcal V)$.

The two integers $k$ and $r$ need not coincide: the carriers are $k$ given
directions, not a basis, and $r=\dim\mathcal C\le k$ with strict inequality
whenever they are linearly dependent. Every subset-indexed statement below
(Sections~\ref{sec:collapse} and~\ref{sec:dissociation}) is therefore stated
relative to a \emph{fixed, arbitrary indexing} $i=1,\ldots,k$ of the given
carriers, which is the indexing an experimenter fixes when choosing which
components to intervene on. We neither assume the $v_i$ independent nor claim
that the results are invariant under a change of spanning set for
$\mathcal C$.

\subsection{Deleting a Subset of Carriers}
\label{sec:deletion}

The intervention analyzed in the rest of the paper deletes a sub-collection of
carriers. For an index subset $S\subseteq\{1,\dots,k\}$, write
\begin{equation}
\label{eq:ablation}
F_S(x) \;:=\; F_0(x) + \sum_{i\notin S} \alpha_i(x)\,v_i ,
\end{equation}
so that the carriers in $S$ contribute nothing while $F_0$ and every other
carrier are left exactly as they were. This is the object
Section~\ref{sec:collapse} analyzes.

Deleting terms is not the same as projecting the output onto
$\mathcal C_S^{\perp}$, where $\mathcal C_S:=\operatorname{span}\{v_i:i\in S\}$:
the projection would also strip from the \emph{surviving} carriers whatever
component they have inside $\mathcal C_S$, and the two operations agree for
every choice of coefficients only when the surviving directions are
orthogonal to $\mathcal C_S$. Taking
$S=\{1,\dots,k\}$ makes the point in the other direction: applying
$I-P_{\mathcal C}$ to the whole sum annihilates it identically, since the sum
lies in $\mathcal C$ by construction, so that extreme case returns $F_0$ and
says nothing about proper subsets. Term deletion, not subspace projection, is
what a weight-space edit of a component performs.

\begin{remark}[From a weight edit to a deleted term]
\label{rem:bridge}
The deletion \eqref{eq:ablation} is exactly what a low-rank edit of one
component's own weights does to the residual stream. In the notation of
Section~\ref{sec:interaction}, an attention head with pre-projection
activation $a(x)$ and output projection $W$ writes $W a(x)$ into the stream;
projecting a unit direction $u$ out of that projection,
$W\mapsto W(I-uu^{\top})$, changes the write-in by exactly
\[
W a(x) \;-\; W(I-uu^{\top})a(x) \;=\; (Wu)\,\langle u,a(x)\rangle,
\]
that is, it deletes one term $v\,\alpha(x)$ with $v:=Wu$ and
$\alpha(x):=\langle u,a(x)\rangle$; a projector of rank $\ell$ deletes $\ell$
such terms, one per direction of its range. This is an identity, not an
approximation. What is \emph{not} an identity, and is the subject of
Section~\ref{sec:interaction}, is the
further assumption that deleting one component's term leaves every other
component's term unchanged.
\end{remark}

\subsection{Selector Representation}

Writing the selector coefficients as $\alpha(x)=(\alpha_1(x),\ldots,\alpha_k(x))
\in\mathbb R^k$ and $V=[v_1,\ldots,v_k]$, the residual computation is
compactly $F(x)=F_0(x)+V\alpha(x)$. This matrix form makes explicit the
separation between the geometry of the supporting directions ($V$) and the
selector responsible for activating them ($\alpha$), the central
abstraction used throughout the remainder of the paper.

\subsection{Standing Assumptions}

The following assumptions frame all subsequent theoretical results.

\begin{assumption}[Residual Additivity]
\label{ass:additivity}
The conditional computation enters the network through an additive residual
connection, so that $F(x)=F_0(x)+\sum_{i=1}^k\alpha_i(x)v_i$ with the $v_i$
independent of $x$.
\end{assumption}

\begin{assumption}[Low-Rank Support]
\label{ass:lowrank}
The conditional component has finite-dimensional support: a
subspace $\mathcal C$ with $\dim(\mathcal C)\ll d$.
\end{assumption}

\begin{assumption}[Stable Base Computation]
\label{ass:stable}
The unconditional component $F_0$ remains invariant under the considered
deletion: $F_0$ in \eqref{eq:ablation} is the same function of $x$ before and
after the intervention.
\end{assumption}

Assumptions~\ref{ass:additivity} and~\ref{ass:stable} are load-bearing:
together they are what licenses writing the ablated computation as
\eqref{eq:ablation}, and they are invoked at that point in every proof of
Section~\ref{sec:collapse}. Assumption~\ref{ass:lowrank} is not. No theorem
below uses $\dim(\mathcal C)\ll d$ quantitatively; the assumption records
\emph{why} one expects a conditional to be carried by a few directions rather
than by an arbitrary linear functional of the stream: it is a modeling
commitment that makes the objects of Section~\ref{sec:deletion} the ones an
experimenter can actually enumerate and intervene on, and we state it as
such rather than pretend the results depend on it. Together the three
assumptions avoid architecture-specific details and define the framework
within which the following theorems are established.

\section{Conditional Collapse}
\label{sec:collapse}

The model of Section~\ref{sec:model} describes how a conditional computation is
embedded into a residual stream. We now ask how the network's discrete decision
responds to a low-rank ablation of that computation. Throughout this section we
restrict attention to \emph{matched pairs}: two inputs intended to receive the
same decision through opposite branches of the conditional. All results below
are conditional on a single additional assumption relating the residual stream
to the network's output.

\begin{assumption}[Linear readout]
\label{ass:readout}
There exist a linear functional $\psi\in(\mathbb R^d)^*$ and a bias
$b\in\mathbb R$ such that the network's binary decision on $F(x)$ is determined
by the sign of
\[
s(x) \;:=\; \psi\bigl(F(x)\bigr) + b,
\]
branch $A$ being selected when $s(x)>0$ and branch $B$ when $s(x)<0$.
\end{assumption}

The boundary case $s(x)=0$ is left undefined, and the same convention applies
to the ablated selector $s_S$ and to the mean margin $\bar s$ introduced
below: we assume throughout that none of these vanishes on the pairs under
consideration. We do not claim this is generic in any sense we have verified;
it is a convention that keeps every statement below a statement about strict
inequalities, and a pair violating it is simply outside the scope of the
results.

\begin{definition}[Matched pair and margins]
\label{def:pair}
A \emph{matched pair} is a pair of inputs $(x_A,x_B)$ for which the network is
correct exactly when
\[
g_A \;:=\; s(x_A) \;>\; 0 \;>\; s(x_B) \;=:\; g_B .
\]
We call $g_A,g_B$ the \emph{margins} of the pair and
$\Phi:=g_A-g_B>0$ its \emph{total contrast}.
\end{definition}

Note that $g_A$, $g_B$ and $\Phi$, and likewise $\bar q_S$, $\Phi_S$,
$\Psi_S$ and $\bar s$ below, are attached to the fixed pair $(x_A,x_B)$ and
carry no free input argument; only quantities genuinely evaluated at a
variable input, such as $\alpha_i(x)$, $s(x)$ and $s_S(x)$, are written with
one.

By Assumptions~\ref{ass:additivity} and~\ref{ass:stable}, ablating a subset
$S\subseteq\{1,\dots,k\}$ of carriers replaces $F$ by $F_S$ of
\eqref{eq:ablation}: the deleted carriers contribute nothing and everything
else, $F_0$ included, is unchanged. Write $\beta_i:=\psi(v_i)$ for the fixed
scalar through which carrier $i$ acts on the readout. The \emph{ablated
selector} is
\[
s_S(x) \;=\; \psi\bigl(F_S(x)\bigr)+b \;=\; s(x)-q_S(x),
\qquad
q_S(x) \;:=\; \sum_{i\in S}\beta_i\,\alpha_i(x),
\]
where $q_S(x)$ is the \emph{removed mass}. On a matched pair, split the removed
mass into symmetric and antisymmetric parts,
\[
\bar q_S \;:=\; \tfrac12\bigl(q_S(x_A)+q_S(x_B)\bigr),
\qquad
\Phi_S \;:=\; q_S(x_A)-q_S(x_B),
\]
so that $q_S(x_A)=\bar q_S+\Phi_S/2$ and $q_S(x_B)=\bar q_S-\Phi_S/2$.

\begin{lemma}[Exact flip criterion]
\label{lem:flip}
For every matched pair and every ablated subset $S$: branch $A$ flips to $B$
if and only if $\bar q_S+\Phi_S/2 > g_A$, and branch $B$ flips to $A$ if
and only if $\bar q_S-\Phi_S/2 < g_B$.
\end{lemma}

\begin{proof}
$s_S(x_A)=g_A-q_S(x_A)=g_A-\bar q_S-\Phi_S/2$; branch $A$ flips exactly
when this is negative. The claim for $B$ follows symmetrically from
$s_S(x_B)=g_B-\bar q_S+\Phi_S/2$.
\end{proof}

Because each carrier is removed exactly rather than through a possibly
imperfect projection, ablating $S$ never leaves a partial residue of the
carriers it contains. What determines collapse is therefore not how
completely $S$ removes its own carriers (always exact here) but
whether $S$ accounts for the \emph{entire} contrast between $x_A$ and $x_B$.
Define the \emph{uncaptured contrast}
\[
\Psi_S \;:=\; \psi\bigl(F_0(x_A)-F_0(x_B)\bigr)
+ \sum_{i\notin S}\beta_i\bigl(\alpha_i(x_A)-\alpha_i(x_B)\bigr),
\]
the portion of $\Phi$ carried by everything $S$ does not touch: the
unconditional part $F_0$ together with every carrier left in place. A direct
computation gives $\Phi_S=\Phi-\Psi_S$: what $S$ removes is exactly
what remains once the uncaptured contrast is subtracted from the total.

\begin{theorem}[Conditional collapse]
\label{thm:collapse}
For a given matched pair $(x_A,x_B)$ and ablated subset $S$, define
\[
\bar s \;:=\; \tfrac12\bigl(g_A+g_B\bigr).
\]
Then $s_S(x_A)=s_S(x_B)=\bar s$ \textup{if and only if}
\begin{enumerate}
\item[(i)] $\bar q_S=0$ \textup{(the removal is symmetric on the pair)}, and
\item[(ii)] $\Psi_S=0$ \textup{(no format contrast survives outside $S$)}.
\end{enumerate}
When (i)--(ii) hold, both inputs of the pair are mapped to the same branch
after ablation: $A$ if $\bar s>0$, $B$ if $\bar s<0$. We call this an
\emph{unconditional collapse} with \emph{polarity}
$\operatorname{sign}\bar s$.
\end{theorem}

\begin{proof}
($\Leftarrow$) Condition (ii) gives $\Phi_S=\Phi$. Substituting into
$q_S(x_A)=\bar q_S+\Phi_S/2$ and $q_S(x_B)=\bar q_S-\Phi_S/2$ together with
condition (i) gives $q_S(x_A)=\Phi/2$ and $q_S(x_B)=-\Phi/2$. Hence
$s_S(x_A)=g_A-\Phi/2=\tfrac12(g_A+g_B)=\bar s$, and
symmetrically $s_S(x_B)=g_B+\Phi/2=\bar s$.

($\Rightarrow$) From $s_S(x_A)=g_A-q_S(x_A)$ and
$s_S(x_B)=g_B-q_S(x_B)$, subtracting gives
$s_S(x_A)-s_S(x_B)=\Phi-\bigl(q_S(x_A)-q_S(x_B)\bigr)=\Phi-\Phi_S=
\Psi_S$; if $s_S(x_A)=s_S(x_B)$ this forces $\Psi_S=0$, which is (ii).
Averaging the same two identities gives
$\tfrac12\bigl(s_S(x_A)+s_S(x_B)\bigr)=\bar s-\bar q_S$; if this common
value equals $\bar s$, then $\bar q_S=0$, which is (i).
\end{proof}

\begin{corollary}[Purity]
\label{cor:purity}
Under the hypotheses of Theorem~\ref{thm:collapse}, if $\bar s\ne0$ then
exactly one input of the pair is misclassified after ablation, and its error
is deterministic: the collapsed decision equals the branch that was already
correct for the \emph{other} input of the pair. No third outcome is possible.
\end{corollary}

\begin{corollary}[Idealized invariance]
\label{cor:invariance}
If two subsets $S,S'$ each satisfy the hypotheses of Theorem~\ref{thm:collapse}
for the same matched pair, then $s_S(x_A)=s_{S'}(x_A)$ and
$s_S(x_B)=s_{S'}(x_B)$: they induce the same collapse, with the same
polarity.
\end{corollary}

\begin{remark}
Corollary~\ref{cor:invariance} is a statement about the idealized model alone:
under Assumption~\ref{ass:readout}, the outcome of any qualifying ablation is
pinned down by the pair's own margins $g_A,g_B$, not by which qualifying
subset realizes it. An empirical instance in which two ablations that each
look complete produce \emph{different} polarities on the same trained network
is therefore evidence that at least one of them fails hypothesis (i) or (ii).
Two mechanisms can produce that failure and they are worth keeping apart: the
true computation may depart from Assumption~\ref{ass:readout}'s additive,
linearly-read picture, which Section~\ref{sec:interaction} identifies and
bounds for one structurally motivated nonlinearity; or the hypotheses may fail
\emph{inside} the linear model, with $\bar q_S$ or $\Psi_S$ simply not small.
Corollary~\ref{cor:robustinvariance} below makes the second alternative
quantitative, and Section~\ref{sec:experiments} finds it is the one that
actually occurs on the instance where we observe a reversal.
\end{remark}

\subsection{A robust form}
\label{sec:robust}

Conditions (i) and (ii) of Theorem~\ref{thm:collapse} are exact equalities
between real numbers computed from a trained network, so they hold on a set of
measure zero and no experiment ever satisfies them. Stated only in that form
the theorem is a statement about an event that never occurs. It is worth
recording that nothing in it depends on the equalities being exact: the same
computation gives an error term, and the error term is an identity rather than
an estimate.

\begin{proposition}[Exact error decomposition and robust collapse]
\label{prop:robust}
For every matched pair and every subset $S$, with no hypothesis whatsoever,
\begin{equation}
\label{eq:robust}
s_S(x_A)-\bar s \;=\; -\bar q_S+\tfrac12\Psi_S,
\qquad
s_S(x_B)-\bar s \;=\; -\bar q_S-\tfrac12\Psi_S .
\end{equation}
Consequently, if $|\bar q_S|\le\varepsilon_1$ and $|\Psi_S|\le\varepsilon_2$,
then
\[
\bigl|s_S(x_A)-\bar s\bigr|\;\le\;\varepsilon_1+\tfrac12\varepsilon_2,
\qquad
\bigl|s_S(x_B)-\bar s\bigr|\;\le\;\varepsilon_1+\tfrac12\varepsilon_2,
\qquad
\bigl|s_S(x_A)-s_S(x_B)\bigr|\;\le\;\varepsilon_2 ,
\]
and if in addition $|\bar s|>\varepsilon_1+\tfrac12\varepsilon_2$, then both
inputs of the pair are mapped to the branch $\operatorname{sign}\bar s$,
exactly one of them is misclassified, and its error is deterministic. That is,
the conclusions of Theorem~\ref{thm:collapse} and Corollary~\ref{cor:purity}
survive verbatim with the exact collapse $s_S(x_A)=s_S(x_B)=\bar s$ weakened to
agreement within $\varepsilon_2$. Theorem~\ref{thm:collapse} is the case
$\varepsilon_1=\varepsilon_2=0$.
\end{proposition}

\begin{proof}
The proof of Theorem~\ref{thm:collapse} already establishes
$s_S(x_A)-s_S(x_B)=\Psi_S$ and
$\tfrac12\bigl(s_S(x_A)+s_S(x_B)\bigr)=\bar s-\bar q_S$; solving these two
linear equations for the individual terms gives \eqref{eq:robust}. (Directly:
$s_S(x_A)-\bar s=g_A-q_S(x_A)-\bigl(g_A-\tfrac12\Phi\bigr)
=\tfrac12(\Phi-\Phi_S)-\bar q_S$, using
$q_S(x_A)=\bar q_S+\tfrac12\Phi_S$ and $\Psi_S=\Phi-\Phi_S$; the computation
for $x_B$ is symmetric.) The three bounds are the triangle inequality applied
to \eqref{eq:robust} and to their difference. For the last claim, write
$z:=s_S(x_A)$; from $|z-\bar s|\le\varepsilon_1+\tfrac12\varepsilon_2<|\bar s|$
we get $z>\bar s-|\bar s|=0$ when $\bar s>0$ and $z<\bar s+|\bar s|=0$ when
$\bar s<0$, so $\operatorname{sign}z=\operatorname{sign}\bar s$, and likewise
for $s_S(x_B)$. Since the pair is matched, the branch
$\operatorname{sign}\bar s$ is the correct answer for exactly one of the two
inputs, which gives purity and determinacy as in
Corollary~\ref{cor:purity}.
\end{proof}

\begin{corollary}[Robust invariance]
\label{cor:robustinvariance}
If $S$ and $S'$ both satisfy $|\bar q|\le\varepsilon_1$ and
$|\Psi|\le\varepsilon_2$ on the same matched pair, then
$\bigl|s_S(x)-s_{S'}(x)\bigr|\le2\varepsilon_1+\varepsilon_2$ for
$x\in\{x_A,x_B\}$, and if $|\bar s|>\varepsilon_1+\tfrac12\varepsilon_2$ they
induce the \emph{same} polarity. Contrapositively, if two subsets collapse the
same pair onto \emph{opposite} branches, they cannot both satisfy the
hypothesis at that tolerance: at least one has
$\varepsilon_1+\tfrac12\varepsilon_2\ge|\bar s|$.
\end{corollary}

\begin{proof}
Subtracting the first identity of \eqref{eq:robust} for $S$ and for $S'$ gives
$s_S(x_A)-s_{S'}(x_A)=(\bar q_{S'}-\bar q_S)+\tfrac12(\Psi_S-\Psi_{S'})$, whose
modulus is at most $2\varepsilon_1+\varepsilon_2$; the computation at $x_B$
differs only in the sign of the second group. The polarity claim and its
contrapositive follow from the last part of Proposition~\ref{prop:robust},
which assigns both subsets the polarity $\operatorname{sign}\bar s$, a
quantity attached to the pair, not to the subset.
\end{proof}

Corollary~\ref{cor:robustinvariance} is the quantitative form of the
observation in the remark above, and it is the one that can be checked: a
polarity reversal between two ablations of the same network is not merely
qualitative evidence that some hypothesis fails, but a lower bound,
$\varepsilon_1+\tfrac12\varepsilon_2\ge|\bar s|$, on how badly it fails for at
least one of them. Section~\ref{sec:experiments} measures
$\bar q_S$, $\Psi_S$ and $\bar s$ on the trained instances and reports how often
the hypothesis of Proposition~\ref{prop:robust} is actually available.

\section{Patching--Ablation Dissociation}
\label{sec:dissociation}

Activation patching and weight-space ablation are often reported side by side
as two ways of probing the same causal object: both are described as
``removing'' or ``restoring'' a carrier's contribution, and both are read as
evidence about whether that carrier is causally load-bearing. We show that,
even in the idealized model of Sections~\ref{sec:model}--\ref{sec:collapse},
the two interventions are governed by different quantities and can disagree
in a specific, predictable way: a carrier can be \emph{sufficient} to flip a
decision under patching while no single-carrier ablation is
\emph{necessary} to flip the same decision. This section makes that
distinction exact.

\begin{definition}[Patching a carrier]
\label{def:patch}
For a matched pair $(x_A,x_B)$ ($x_A$ the \emph{donor}, $x_B$ the
\emph{receiver}) and a single carrier $i$, \emph{patching} $i$ replaces
$\alpha_i(x_B)$ by $\alpha_i(x_A)$ inside the receiver's computation, leaving
every other carrier and $F_0$ evaluated at $x_B$. The patched selector is
\[
s_i^{\to B}(x_B) \;:=\; s(x_B) + \beta_i\delta_i,
\qquad
\delta_i \;:=\; \alpha_i(x_A)-\alpha_i(x_B).
\]
\end{definition}

\begin{theorem}[Patching--ablation dissociation]
\label{thm:dissociation}
For every carrier $i$ and matched pair $(x_A,x_B)$:
\begin{enumerate}
\item[(a)] Patching $i$ into the receiver changes the selector by exactly
$\beta_i\delta_i$, \emph{independently of $F_0$, of every other carrier,
and of how the total contrast $\Phi$ is distributed among carriers};
the receiver flips if and only if $\beta_i\delta_i > -g_B$.
\item[(b)] Ablating $i$ alone changes the selector by exactly $-\beta_i\alpha_i(x_B)$,
a quantity that depends on the receiver's own coefficient $\alpha_i(x_B)$, not
on the contrast $\delta_i$; the receiver flips if and only if
$-\beta_i\alpha_i(x_B) > -g_B$.
\end{enumerate}
In particular, patching sufficiency and ablation necessity are governed by two
generally unrelated quantities, $\beta_i\delta_i$ and $\beta_i\alpha_i(x_B)$,
and neither bounds the other without further assumptions.
\end{theorem}

\begin{proof}
Part (a) is Definition~\ref{def:patch} together with the sign condition for
branch $B$ in the sense of Assumption~\ref{ass:readout}. Part (b) is
Lemma~\ref{lem:flip} specialized to $S=\{i\}$, for which $\bar q_{\{i\}}
=\tfrac12(\beta_i\alpha_i(x_A)+\beta_i\alpha_i(x_B))$ and $\Phi_{\{i\}}
=\beta_i\delta_i$, giving $\bar q_{\{i\}}-\Phi_{\{i\}}/2=\beta_i\alpha_i(x_B)$
after cancellation.
\end{proof}

\begin{corollary}[Sufficiency without necessity]
\label{cor:sufficiency}
There exist matched pairs and families of $n\ge2$ carriers for which every
single-carrier patch flips the receiver while no single-carrier ablation
does.
\end{corollary}

\begin{proof}[Proof by explicit construction]
Fix $n\ge2$, set $\beta_i=1$, $\alpha_i(x_B)=-\tfrac12$ and $\alpha_i(x_A)=1$
for $i=1,\dots,n$ (so $\delta_i=\tfrac32$), and choose $F_0$ with
$\psi(F_0(x_A))=\psi(F_0(x_B))$ (e.g.\ $F_0\equiv0$) and $b$ so that
$g_B=-1$. Every single patch satisfies $\beta_i\delta_i=\tfrac32>1=-g_B$
and therefore flips the receiver by part (a). Every single ablation satisfies
$-\beta_i\alpha_i(x_B)=\tfrac12<1=-g_B$ and therefore does not, by part
(b). Because $\psi(F_0(x_A))=\psi(F_0(x_B))$, the margins are exactly
$g_A=-1+\tfrac32 n$ and $g_B=-1$, so consistency ($g_A>0$) holds for
every $n\ge2$, and the construction exhibits a family of $n\ge2$ carriers
satisfying both requirements simultaneously.
\end{proof}

\begin{remark}[Discussion]
\label{rem:redundancy}
The mechanism behind Corollary~\ref{cor:sufficiency} is redundancy, not an
artifact of the construction: if $n$ carriers each carry enough
\emph{contrast} $\delta_i$ to flip the receiver on their own, patching any one
of them injects that full contrast and flips it: activation-level
sufficiency requires nothing about the other $n-1$ carriers. Ablating one of
the $n$ carriers, in contrast, removes only its own \emph{absolute} level
$\alpha_i(x_B)$ at the receiver; if the receiver's correct decision is instead
supported by the combination of what remains (the untouched carriers and
$F_0$), removing one contributor changes nothing about whether the others
still suffice.

This is the regime that produces a patch-recovery ratio $\kappa>1$, i.e.\ a
single-site patch that \emph{overshoots} the donor's output rather than merely
reaching it: the patched contrast exceeds what is strictly required to cross
the margin, precisely because the same margin is, in the unablated network,
being protected redundantly from the ablation side. Section~\ref{sec:experiments}
reports both halves of this signature on the same trained network (a
single-site patch with $\kappa>1$ at a site whose weight-level ablation
nonetheless leaves the majority of receiver-format inputs answered correctly),
which is an instance of Theorem~\ref{thm:dissociation}, not a separate
phenomenon. It is also, we note, what an observer with access to only one of
the two interventions would report as an anomaly: overshooting recovery from
the patching side, or an apparently dispensable component from the ablation
side.
\end{remark}

\section{Nonlinear Interaction Theorem}
\label{sec:interaction}

Sections~\ref{sec:collapse} and~\ref{sec:dissociation} treat each carrier's
contribution $\alpha_i(x)v_i$ as exact and independent of every other
carrier: ablating a set $S$ simply deletes the corresponding terms from the
sum, leaving the rest untouched. This section identifies precisely when that
independence itself is exact, and exhibits the architectural mechanism through
which two carriers of the same residual block fail to satisfy it.

\subsection{A concrete two-carrier composition}

Vectors are columns throughout this section, and every operator acts on the
left. Consider two carriers written into the same residual stream by a single
block. The first is an attention head with pre-projection activation
$a(x)\in\mathbb R^{d_{\mathrm h}}$ and output projection
$W\in\mathbb R^{d\times d_{\mathrm h}}$, contributing $W a(x)$ to the
post-attention residual $r_1(x):=x+\mathrm{ao}(x)$. The second is the block's
own MLP $M$, applied \emph{after} a normalization layer $N$ acting on $r_1$.
Writing $g:=M\circ N$ for the composition of the normalization and the MLP
(this $g$ is unrelated to the pair margins $g_A,g_B$ of
Section~\ref{sec:collapse}; context disambiguates throughout), the second
carrier's realized value is $g(r_1(x))$, a function \emph{of the first
carrier's output}, not an independent term.

Ablating the head with an orthogonal projector $P=UU^\top$ on $\mathbb R^{d_{\mathrm h}}$
replaces $W$ by $W(I-P)$, equivalently subtracting $\eta(x):=W P a(x)$ from
$r_1$ (Remark~\ref{rem:bridge}); ablating the MLP itself with an orthogonal projector $Q$
on $\mathbb R^{d}$ replaces $M(z)$ by $(I-Q)M(z)$. A
\emph{frozen-activation} estimate of the resulting change in the MLP's
contribution (treating the two carriers as independent, in the sense of
Section~\ref{sec:collapse}) computes $g(r_1(x))$ once, applies $(I-Q)$ if
the MLP itself is ablated, and otherwise leaves the head's ablation with no
effect on it; the \emph{true} computation instead recomputes the MLP from
whatever residual the head's ablation actually produces, then applies
$(I-Q)$. Their difference,
\[
\Delta(x) \;:=\; (I-Q)\Bigl[g\bigl(r_1(x)-\eta(x)\bigr) - g\bigl(r_1(x)\bigr)\Bigr],
\]
is exactly the interaction that the independent-carrier model of
Sections~\ref{sec:model}--\ref{sec:dissociation} sets to zero by assumption.
We now derive $\Delta(x)$ exactly and bound it.

\subsection{The interaction term}

Let $N(r):=\rho(r)\,\gamma\odot r$ with $\rho(r):=(\tfrac1d\|r\|^2+\varepsilon)^{-1/2}$
denote root-mean-square normalization with scale $\gamma$ and $\varepsilon>0$.

\begin{lemma}[Normalization Jacobian]
\label{lem:jacobian}
$N$ is $C^\infty$ on all of $\mathbb R^d$, with
\[
DN(r) \;=\; \rho(r)\operatorname{diag}(\gamma)
\Bigl(I_d - \tfrac{\rho(r)^2}{d}rr^\top\Bigr),
\qquad
\|DN(r)\|_{\mathrm{op}} \;\le\; \|\gamma\|_\infty\,\rho(r).
\]
\end{lemma}

\begin{proof}
With $T(r):=\tfrac1d\|r\|^2+\varepsilon\ge\varepsilon>0$, $\rho=T^{-1/2}$ is
smooth everywhere and $\partial\rho/\partial r_j=-\tfrac1d\rho^3 r_j$. Since
$N_i(r)=\gamma_i\rho(r)r_i$,
$\partial N_i/\partial r_j=\gamma_i\rho\bigl(\delta_{ij}-\tfrac{\rho^2}{d}r_ir_j\bigr)$,
which is the displayed matrix. The matrix $I_d-\tfrac{\rho^2}{d}rr^\top$ is
symmetric with eigenvalue $\tfrac{d\varepsilon}{\|r\|^2+d\varepsilon}\in(0,1]$
along $r$ and eigenvalue $1$ on $r^\perp$, hence operator norm $1$; combined
with $\|\operatorname{diag}(\gamma)\|_{\mathrm{op}}=\|\gamma\|_\infty$ this
gives the bound.
\end{proof}

\begin{theorem}[First-order interaction formula]
\label{thm:interaction}
For every $x$,
\[
\Delta(x) \;=\; (I-Q)\Bigl[-Dg\bigl(r_1(x)\bigr)\,\eta(x) \;-\; R(x)\Bigr],
\]
\[
R(x) \;:=\; \int_0^1\Bigl[Dg\bigl(r_1(x)-t\eta(x)\bigr)-Dg\bigl(r_1(x)\bigr)\Bigr]\eta(x)\,dt,
\]
where $Dg=DM(N(r))\,DN(r)$ by the chain rule. The identity is exact for every
$x$, every $\eta(x)$, and every $Q$. Moreover $g$ is $C^\infty$
(Lemma~\ref{lem:jacobian} for $N$; $M$ is a composition of linear maps and
smooth elementwise nonlinearities), so $\|R(x)\|\le L(x)\|\eta(x)\|$ with
$L(x):=\sup_{t\in[0,1]}\|Dg(r_1(x)-t\eta(x))-Dg(r_1(x))\|_{\mathrm{op}}$
finite; if in addition $\|D^2g\|_{\mathrm{op}}\le\Lambda$ on the segment
joining $r_1(x)-\eta(x)$ and $r_1(x)$, then $\|R(x)\|\le\Lambda\|\eta(x)\|^2$:
the interaction is second order in the size of the perturbation.
\end{theorem}

\begin{proof}
Let $\phi(t):=g(r_1(x)-t\eta(x))$, so $\phi'(t)=-Dg(r_1(x)-t\eta(x))\eta(x)$
and, by the fundamental theorem of calculus,
$g(r_1(x)-\eta(x))-g(r_1(x))=\phi(1)-\phi(0)=\int_0^1\phi'(t)\,dt$. Adding and
subtracting $Dg(r_1(x))\eta(x)$ inside the integral and left-multiplying by
$(I-Q)$ gives the displayed identity. The bound on $R(x)$ follows from
$\|R(x)\|\le\sup_t\|Dg(r_1(x)-t\eta(x))-Dg(r_1(x))
\|_{\mathrm{op}}\|\eta(x)\|$; if $D^2g$ is bounded by $\Lambda$ on the segment,
the mean value inequality gives
$\|Dg(r_1(x)-t\eta(x))-Dg(r_1(x))\|_{\mathrm{op}}\le\Lambda t\|\eta(x)\|
\le\Lambda\|\eta(x)\|$ for all $t\in[0,1]$, hence $L(x)\le\Lambda\|\eta(x)\|$.
\end{proof}

\begin{corollary}[Single-carrier exactness]
\label{cor:exactness}
$\Delta(x)\equiv0$ identically, for every $x$, if $\eta(x)\equiv0$ (in
particular whenever only the MLP, and not the head, is ablated), since both
terms inside the bracket of Theorem~\ref{thm:interaction} then vanish. The
same holds trivially if $Q=I$, since the prefactor $(I-Q)$ annihilates the
bracket regardless of $\eta$ (a case of no practical interest under
Assumption~\ref{ass:lowrank}, recorded only for completeness). Ablating the
head \emph{alone} ($Q=0$, generically $\eta(x)\ne0$) satisfies neither
condition: Theorem~\ref{thm:interaction} still applies and bounds
$\Delta(x)=g(r_1(x)-\eta(x))-g(r_1(x))$, but nothing forces this quantity to
vanish, and it generically does not: ablating a head changes its own
layer's MLP output whether or not the MLP is separately ablated.
Proposition~\ref{prop:readout} below makes precise what this costs the
idealized model's own prediction in that case.
\end{corollary}

\begin{proof}
If $\eta(x)=0$: the bracket in Theorem~\ref{thm:interaction} is $0-0=0$
independently of $Q$. If $Q=I$: the prefactor $(I-Q)$ is the zero operator,
independently of the bracket. Neither is implied by $Q=0$ alone.
\end{proof}

Corollary~\ref{cor:exactness} concerns $\Delta(x)$, the quantity the idealized
model omits. A different single-carrier exactness statement, easily confused
with it, concerns instead the \emph{measurement protocol} used to estimate
$\Delta$ empirically, and holds for a head just as much as for an MLP.

\begin{proposition}[Single-carrier patch--edit equivalence]
\label{prop:patchedit}
Fix an input $x$ and a single component whose write-in to the residual stream
is $Wc(x)$, with $c(x)$ its own activation and $W$ its output matrix. Ablate
it by an orthogonal projector, applied either on the activation side ($P$, so
that the intended write-in is $W(I-P)c(x)$) or on the stream side ($Q$, giving
$(I-Q)Wc(x)$). Then the \emph{weight-edit} route (edit $W$ and run a fresh
forward pass) and the \emph{frozen-activation} route (leave all weights
intact, replace this component's realized output by its projected value, and
recompute everything downstream) assign identical values to every node of
the network.
\end{proposition}

\begin{proof}
$c(x)$ is computed strictly upstream of $W$, so editing $W$ does not change
it. On the activation side, applying the projector to the activation gives
write-in $W\bigl((I-P)c(x)\bigr)$ and applying it to the matrix gives
$\bigl(W(I-P)\bigr)c(x)$; these are the same vector. On the stream side,
$(I-Q)\bigl(Wc(x)\bigr)=\bigl((I-Q)W\bigr)c(x)$ likewise. The two routes
therefore agree on this component's write-in, every other weight is untouched,
and hence every downstream node is computed from identical inputs by identical
maps.
\end{proof}

\begin{remark}
\label{rem:patchedit-estimation}
Proposition~\ref{prop:patchedit} holds for \emph{any} projector, in particular
for one estimated from data: both routes use the same projector, so an error
in choosing it is common to the two and cancels in their difference. A
measured gap between the two routes on a single-carrier ablation is therefore
neither evidence about $\Delta(x)$ nor evidence about how well the ablated
direction was estimated, a point Section~\ref{sec:experiments} returns to
with measurements.
\end{remark}

\begin{proposition}[Propagation to the readout]
\label{prop:readout}
Suppose that the block of
Section~\ref{sec:interaction} feeds the readout directly
(Remark~\ref{rem:propagation} discusses what is and is not available when it
does not), i.e.\
$F(x)=r_1(x)+g(r_1(x))+E(x)$ where $E(x)$ collects $F_0(x)$ and every carrier
not drawn from this block. Let $S\subseteq\{1,\dots,k\}$ consist only of
carriers drawn from this block's own head and/or its own MLP (so that no
carrier in $S$ lies upstream of $r_1(x)$ and $E(x)$ is unaffected by
ablating $S$): $\eta(x)$ is the head's contribution if the head's carrier is
in $S$ and $0$ otherwise, and $Q$ is the MLP's ablation projector if the
MLP's carrier is in $S$ and $0$ otherwise (Remark~\ref{rem:bridge} identifies
both as genuine elements of $\{v_1,\dots,v_k\}$). Then the idealized
selector $s_S(x)$ of Section~\ref{sec:collapse} and the selector
$s_S^{\mathrm{true}}(x)$ of the genuinely weight-edited network agree up to
exactly the interaction term of Theorem~\ref{thm:interaction}:
\[
s_S^{\mathrm{true}}(x) \;=\; s_S(x) \;+\; \psi\bigl(\Delta(x)\bigr), \qquad
\bigl|\psi(\Delta(x))\bigr| \;\le\; \|\psi\|_{\mathrm{op}}\Bigl(\|Dg(r_1(x))\|_{\mathrm{op}}\|\eta(x)\|+L(x)\|\eta(x)\|\Bigr),
\]
with $L(x)$ as in Theorem~\ref{thm:interaction}, and
$\bigl|\psi(\Delta(x))\bigr|\le\|\psi\|_{\mathrm{op}}\bigl(\|Dg(r_1(x))\|_{\mathrm{op}}\|\eta(x)\|+\Lambda\|\eta(x)\|^2\bigr)$
under that theorem's additional $\|D^2g\|_{\mathrm{op}}\le\Lambda$ hypothesis.
In particular $s_S^{\mathrm{true}}(x)=s_S(x)$ exactly whenever
Corollary~\ref{cor:exactness}'s sufficient conditions hold, and
Lemma~\ref{lem:flip} and Theorem~\ref{thm:collapse}, stated for $s_S$,
transfer to the true network without error in that case; otherwise they
transfer with an error bounded as above, in particular whenever only the
head, and not the MLP, is ablated.
\end{proposition}

\begin{proof}
By Assumptions~\ref{ass:additivity} and~\ref{ass:stable}, $E(x)$ is the same
function of $x$ in the clean network and in $F_S(x)$; by the hypothesis that
$S$ draws only from this block, it is also unaffected by the true weight
edit. The idealized model \eqref{eq:ablation} deletes the head's carrier
exactly and the MLP's carrier by subtracting its removed mass computed on
the \emph{clean} residual $r_1(x)$: every surviving carrier's coefficient,
the MLP's included, is by \eqref{eq:ablation} the same function of $x$ as in
the unablated network:
\[
F_S(x) \;=\; F(x)-\eta(x)-Qg(r_1(x)) \;=\; \bigl(r_1(x)-\eta(x)\bigr) + (I-Q)g\bigl(r_1(x)\bigr) + E(x).
\]
The true, weight-edited network instead recomputes the MLP on the genuinely
ablated residual before applying its own ablation:
\[
F_S^{\mathrm{true}}(x) \;=\; \bigl(r_1(x)-\eta(x)\bigr) + (I-Q)g\bigl(r_1(x)-\eta(x)\bigr) + E(x).
\]
Subtracting, $F_S^{\mathrm{true}}(x)-F_S(x) = (I-Q)\bigl[g(r_1(x)-\eta(x))-g(r_1(x))\bigr]=\Delta(x)$
by definition. Applying $\psi$ and adding $b$ gives
$s_S^{\mathrm{true}}(x)=s_S(x)+\psi(\Delta(x))$; the stated bounds follow
from Theorem~\ref{thm:interaction}'s bounds on $\|\Delta(x)\|$,
$|\psi(\Delta(x))|\le\|\psi\|_{\mathrm{op}}\|\Delta(x)\|$, and
$\|I-Q\|_{\mathrm{op}}\le1$ since $Q$ is an orthogonal projector.
\end{proof}

\begin{remark}[Interpretation and limits]
Corollary~\ref{cor:exactness} and Proposition~\ref{prop:readout} together
close the loop opened by Corollary~\ref{cor:invariance} in
Section~\ref{sec:collapse}: two ablated subsets can satisfy the idealized
collapse hypotheses exactly, and hence agree with the true network, whenever
$\psi(\Delta(x))=0$ for each, which $\eta(x)\equiv0$ (an MLP-alone ablation)
guarantees; otherwise the idealized selector and the true one differ by
exactly $\psi(\Delta(x))$, which is bounded and generically nonzero, though
not necessarily so at every $x$: $\Delta(x)$ may happen to lie in $\ker\psi$
there. Ablating a head at all, whether or not its own layer's MLP
is separately ablated, therefore generically (though not necessarily)
perturbs the idealized prediction; only an MLP-alone ablation is seen exactly
for every $x$. Two limitations bound the scope of this result and are not
closed here (a third, the absence of a closed form for the curvature
constant, is closed in a companion analysis): Proposition~\ref{prop:readout}
tracks $\Delta(x)$'s effect on the selector exactly only when this block
feeds the readout directly, and propagating it through further nonlinear
layers when it does not (the general case Remark~\ref{rem:propagation}
discusses only at the level of $\|\Delta(x)\|$, not of $\psi(\Delta(x))$)
remains open; and the two-carrier composition considered here (one attention
head and its own layer's MLP) does not by itself cover interactions between
carriers in different layers. A companion analysis shows exactly how far the
same technique reaches into that case: the same-block discrepancy at each
touched layer is isolated exactly, in closed form only at the layer that
feeds the readout directly, and what is left over (both the propagated
same-block effect at every other touched layer and the genuinely cross-layer
part) is isolated, not hidden, as a single separately measurable
remainder.
\end{remark}

\begin{remark}[Propagation across layers]
\label{rem:propagation}
The analysis above stops at the output of one block; in a deep network,
$\Delta(x)$ is written into the residual stream and carried forward by every
subsequent layer. Two effects govern its fate, both already implicit in the
tools developed here. Every subsequent normalization layer acts on it through
the same Jacobian bound as Lemma~\ref{lem:jacobian}: since $\rho(r)$ shrinks
as $\|r\|$ grows, normalization attenuates a fixed-size perturbation more
strongly on residual streams of larger norm, one layer later. Every
subsequent linear map, in turn, can amplify it by up to its own operator
norm. A perturbation surviving $L$ further layers is therefore bounded,
heuristically, by a product of $L$ alternating attenuation and amplification
factors, none of which is guaranteed to stay below one. Whether this product
typically grows, shrinks, or is absorbed by the residual stream's own
accumulation on trained weights is an empirical question the present model
does not answer. A companion analysis turns this into a sharper question than
``how does a perturbation propagate'': it gives an exact decomposition that
isolates the entire cross-layer effect, for an arbitrary ablated subset
spanning any number of layers, as one remainder term separate from the
same-block discrepancy at each touched layer, closed form only at the layer
feeding the readout directly and pinned to zero at any layer whose own
carriers are MLP-only, without bounding either open piece there either, which
we leave, along with the present question, for future work.
\end{remark}

\section{Illustrative Experiments}
\label{sec:experiments}

The theorems of Sections~\ref{sec:collapse}--\ref{sec:interaction} are
statements about an idealized model. This section checks that their
predictions (collapse under joint ablation, patch/ablation dissociation,
and an interaction term that vanishes when the MLP alone is ablated but not,
in general, otherwise) are realized in an actual trained network, rather
than being vacuous under the stated assumptions. We report the minimum
needed to make this check, not a general empirical study of the underlying
task.

\paragraph{Task and vocabulary.} We reuse the setting and the measurements of
the empirical study these theorems were written to explain, restated here in
the minimum detail needed to read the table and the two figures. Small
transformers are trained on a synthetic associative-recall task: a context of
key--value pairs is followed by a two-token query, a marker $m_A$ or $m_B$
and a displayed token, and the answer is the value $v(k)$ of the true key $k$.
The marker selects the surface format. Under $m_A$ the displayed token
\emph{is} the true key; under $m_B$ it is $\sigma(k)$ for a fixed derangement
$\sigma$, so the network must apply $\sigma^{-1}$ before looking the value up.
Both formats demand the same answer, so a competent network implements exactly
a conditional of the kind formalized in Section~\ref{sec:model}, and the two
formats are the two branches $A$ and $B$ of Definition~\ref{def:pair}.

Each branch has one \emph{systematic} wrong answer, which is what makes a
collapse observable rather than merely a drop in accuracy. On format $A$ the
wrong reading is the \emph{inverted} one, $v(\sigma^{-1}(k))$: the network
applies $\sigma^{-1}$ where it should not. On format $B$ it is the
\emph{literal} one, $v(\sigma(k))$: the network reads the displayed token
at face value. All rates below are measured on filtered evaluation
distributions on which the wrong reading is a well-defined token distinct from
the correct one, so ``inverted-$A$ rate $0.93$'' means the ablated network
returned that specific token on $93\%$ of filtered format-$A$ inputs.

\paragraph{Carriers and ablation configurations.} Across five independently
trained instances, a greedy activation-patching search identifies a small set
of components (\emph{carriers}, in the sense of Section~\ref{sec:model})
whose activation swap between formats flips the answer, a redundant code
of the kind reported elsewhere in superposed and self-repairing networks
\cite{elhage2022,mcgrath2023}. Each carrier is a single attention head or a
single MLP, i.e.\ the concrete realization of a direction $v_i$. We then
ablate carriers directly in weight space, projecting the low-rank subspace of
their donor--receiver activation deltas out of the corresponding weight
matrices, exactly as in Remark~\ref{rem:bridge}: an instance of the deletion
\eqref{eq:ablation} with the directions estimated from data. Three subsets
$S$ are used per instance: \textbf{D1}, the single strongest carrier alone;
\textbf{DJ}, all first-layer carriers jointly (two of them, in every instance
here); and \textbf{DJA}, all carriers, including those in deeper layers. The
three are nested, $S_{\mathrm{D1}}\subset S_{\mathrm{DJ}}\subseteq
S_{\mathrm{DJA}}$. Seed 11 has
only two carriers in total and therefore has no DJA configuration, which is
why the five instances yield $14$ configurations rather than $15$.

\begin{table}[htbp]
\centering
\small
\begin{tabular}{lccc}
\toprule
Instance & Total rank $\sum k$ & Best inverted-$A$ rate & Best literal-$B$ rate \\
\midrule
inst.\ 2  & $5$ & $.86$ & $.01$ \\
seed 11   & $6$ & $.41$ & $.47$ \\
seed 22   & $4$ & $.49$ & $.27$ \\
seed 33   & $5$ & $.87$ & $.20$ \\
seed 44   & $4$ & $.93$ & $.977$ \\
\bottomrule
\end{tabular}
\caption{Redundancy ($\ge2$ carriers per instance) and low rank ($\sum k\le6$)
hold on $5/5$ instances. The two rightmost columns report, for each instance,
the highest rate at which ablation collapses the network onto each of the two
unconditional readers of Theorem~\ref{thm:collapse} (the inverted reader on
format $A$, the literal reader on format $B$) over the configurations D1,
DJ and DJA. Each is therefore a maximum over the $2$--$3$ configurations
available for that instance, not a single measurement, and should be read as
``some tested subset achieves this'' rather than as a typical value.
Rates are measured on the filtered distributions; the un-ablated
baselines checked on these instances sit at $0$ there (seed 44's is the
leftmost group of Figure~\ref{fig:seed44}), so each entry is a collapse rate
over a zero baseline, not a difference against a nonzero one. Seed 44 is the
only instance in which \emph{both} collapses are realized on the same trained
weights, each with the untouched format retained at rate $\ge0.96$, by two
\emph{nested} subsets $S$.}
\label{tab:results}
\end{table}

Table~\ref{tab:results} illustrates Theorem~\ref{thm:collapse} directly: on
four of five instances, ablation collapses the network onto one unconditional
reader but not, at the tested subsets, the other. Seed 44 is the exception
that matters. There D1 deletes a single carrier, the rank-one head $h3$, and
collapses the network onto the inverted reader (inverted-$A$ rate $0.93$,
format $B$ retained at $0.963$); DJ deletes $h3$ \emph{together with} the same
block's MLP and collapses it onto the literal reader (literal-$B$ rate
$0.977$, format $A$ retained at $0.997$). The two subsets are nested,
$S_{\mathrm{D1}}\subset S_{\mathrm{DJ}}$: enlarging a subset that already
produces a clean collapse does not deepen that collapse but reverses its
polarity. By Corollary~\ref{cor:invariance} this cannot happen inside the
idealized model, so at least one of the two subsets must violate hypothesis
(i) or (ii). Both subsets ablate the head $h3$, so by
Corollary~\ref{cor:exactness} neither is guaranteed the exact agreement
Proposition~\ref{prop:readout} reserves for an MLP-alone ablation; DJ adds
the same block's MLP on top, giving Theorem~\ref{thm:interaction}'s
interaction term every freedom to differ in size and sign between the two
configurations, which is sufficient to flip which branch the idealized
hypotheses would otherwise identify.

\begin{figure}[htbp]
\centering
\includegraphics[width=\linewidth]{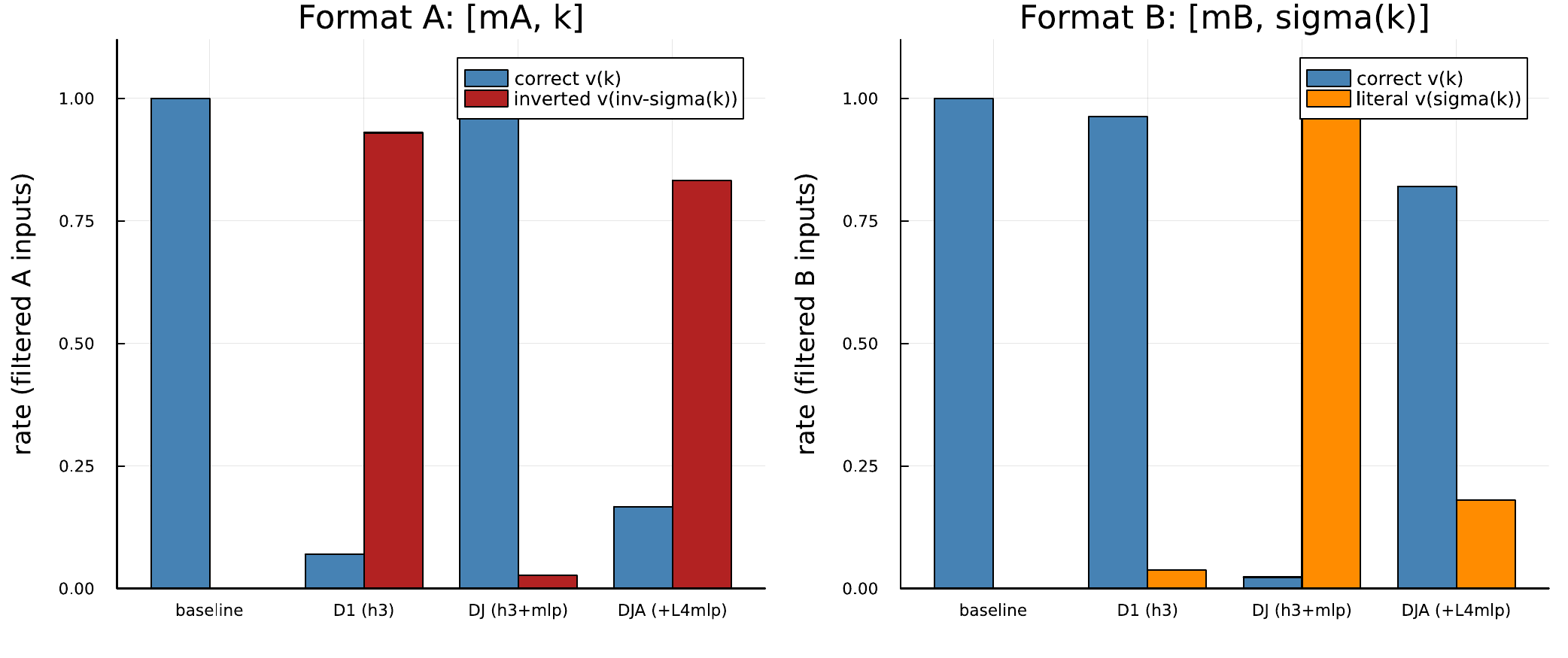}
\caption{Seed 44: two \emph{nested} ablated subsets
($S_{\mathrm{D1}}\subset S_{\mathrm{DJ}}$) collapse the same trained network
onto \emph{opposite} unconditional readers, each with the untouched format
retained at rate $\ge0.96$. Left: filtered format-$A$ inputs, correct answer
$v(k)$ against the inverted reading $v(\sigma^{-1}(k))$. Right: filtered
format-$B$ inputs, correct answer against the literal reading $v(\sigma(k))$.
The leftmost group of each panel is the un-ablated baseline. This is the
phenomenon Corollary~\ref{cor:invariance} shows cannot occur within the
idealized model, and that Theorem~\ref{thm:interaction} traces to an
interaction term that Corollary~\ref{cor:exactness} guarantees only an
MLP-alone ablation avoids: neither D1 nor DJ qualifies for that guarantee
here, since both ablate the head.}
\label{fig:seed44}
\end{figure}

\paragraph{Patching versus ablation on the same site.} The dissociation of
Theorem~\ref{thm:dissociation} is visible on a further instance of the same
task, on which the patch search returns a \emph{single} sufficient site (the
first-layer MLP output). Patching that site alone gives a median recovery of
$\kappa=1.043$ over $20$ independent validation pairs (values $1.04$--$1.22$
on the pairs used for the search) and an argmax flip rate of $0.85$: recovery
above $1$, the overshoot signature of Remark~\ref{rem:redundancy}. Ablating the weights
feeding the same site leaves the receiver format correct on $57\%$ of inputs:
patch-sufficient, but far from ablation-necessary. The error it does
produce is nonetheless the predicted one: after ablation, correct plus
literal accounts for $0.993$ of the answers, so what the ablation removes is
the conditional and not the associative lookup. Sufficiency under patching
and necessity under ablation are measured on the same component of the same
network and disagree, exactly as $\beta_i\delta_i$ and
$\beta_i\alpha_i(x_B)$ are permitted to.

\paragraph{Are the collapse hypotheses ever available?} Theorem~\ref{thm:collapse}'s
conditions are exact equalities and hold on a null set, so the question that
can actually be asked of a trained network is the one
Proposition~\ref{prop:robust} poses: are the tolerances small enough, relative
to $|\bar s|$, to certify anything? We measured $\bar q_S$, $\Psi_S$ and
$\bar s$ per matched pair on all $39$ configurations ($120$ pairs each, on
probe seeds again disjoint from every other measurement here). The identities
\eqref{eq:robust} reproduce to $0$ in floating point on all $4680$
pair-configurations, which checks the algebra rather than the network.

The certificate is rarely available. Taking $\varepsilon_1=|\bar q_S|$ and
$\varepsilon_2=|\Psi_S|$ at their measured per-pair values, the condition
$|\bar s|>\varepsilon_1+\tfrac12\varepsilon_2$ holds on $10.0\%$ of
pair-configurations; no configuration satisfies it on more than $71.7\%$ of its
pairs, and none on $90\%$. The reason is structural rather than incidental.
$\bar s$ is the \emph{mean} of a positive and a negative margin on a matched
pair, so it is small by construction (median $|\bar s|/\Phi=0.083$), while
the two tolerances are of the order of the margins themselves, median
$|\bar q_S|/\Phi=0.14$ and $|\Psi_S|/\Phi=0.21$. In the median the sufficient
condition misses by a factor of $5.3$.

It is nonetheless a \emph{sufficient} condition, and a conservative one: the
conclusion it certifies (both selectors on the side of $\bar s$) holds on
$44.5\%$ of pair-configurations, four times more often than the certificate
fires, and no configuration was certified on at least half its pairs while
collapsing on fewer than half, as Proposition~\ref{prop:robust} requires
pairwise. It is also discriminating rather than uniformly pessimistic:
certificate rate and realized collapse rate have Spearman $\rho=0.66$ across
the $39$ configurations ($p\approx1\times10^{-5}$), and mean collapse rate
rises from $32\%$ to $57\%$ to $88\%$ across the configurations whose
certificate rate is below $5\%$, between $5\%$ and $25\%$, and above $25\%$.
On seed 44 it separates precisely the two configurations of
Figure~\ref{fig:seed44}: D1, the clean single-carrier collapse, is the
best-certified configuration in the entire set ($71.7\%$ of pairs certified,
$96.7\%$ collapsing), whereas DJ, the nested subset that reverses polarity, is
certified on $0.8\%$, with $|\bar q_S|=6.3$ against $|\bar s|=4.2$: exactly
the failure of hypothesis (i) that Corollary~\ref{cor:robustinvariance} shows
any polarity reversal must exhibit.

We report this as a mixed result and prefer not to round it in either
direction. The robust form is what makes Theorem~\ref{thm:collapse} applicable
to a real network at all, since the exact hypotheses never hold; but on these
instances it certifies a minority of configurations, and its practical value
here is as a diagnostic that ranks configurations by how close they are to the
idealized regime, not as a guarantee covering the behavior we observe.

\paragraph{The interaction term.} Finally we test Corollary~\ref{cor:exactness}
and Proposition~\ref{prop:readout} against the same $14$ configurations. For
each we measure two quantities. The \emph{interaction magnitude} is the
largest gap, over the probed inputs and in logit units, between the
frozen-activation prediction of the ablated selector and the selector of the
network whose weights were actually edited. For a \emph{jointly} ablated
head and MLP, the frozen-activation prediction holds the MLP's surviving
output at its \emph{clean-run} value (exactly the bookkeeping
Assumption~\ref{ass:stable} prescribes for $F_S(x)$), while the true
network recomputes it from the head-ablated residual; the gap between the
two is exactly what Proposition~\ref{prop:readout} predicts, $\psi(\Delta(x))$.
For a \emph{single} ablated carrier, the standard frozen-activation protocol
instead estimates the removed mass by patching that one carrier's own
contribution and letting the rest of the network recompute naturally, and
by Proposition~\ref{prop:patchedit} this reproduces the true weight edit node
for node, for either carrier alone and regardless of $\Delta$. The gap being
$\approx0$ on a single-carrier configuration is therefore a fact about that
measurement protocol, not, by itself, evidence that
$\Delta(x)=0$ there; Corollary~\ref{cor:exactness} guarantees
$\Delta(x)\equiv0$ only for the single-carrier configurations that ablate the
MLP. The \emph{agreement rate} (C1a in
Figure~\ref{fig:interaction}) is the fraction of inputs on which the idealized
model's predicted answer matches the edited network's actual answer; a second,
coarser criterion (C1b) asks whether the model's predicted inverted-$A$ and
literal-$B$ \emph{rates} match the observed ones to within $\pm0.10$. Both
thresholds ($0.90$ and $\pm0.10$) were fixed in the verification script before
any checkpoint was probed.

\paragraph{Is the single-carrier coincidence exact?} The argument above says
the frozen-activation protocol and the weight edit should agree exactly on a
single-carrier configuration, but the directions being removed are estimated
from data, and C1a on those configurations is not exactly $1$. Whether the
residual is the estimation error surfacing is a question about the pipeline's
own numerics, so we measured it rather than argued it, on $400$ held-out probe
inputs per instance drawn by the procedure above but with seeds distinct from
those behind any other number reported here. Over all five D1 configurations
the per-input gap between the frozen-activation selector and the selector of
the genuinely weight-edited network has median at most $9.6\times10^{-7}$ and
maximum $7.7\times10^{-6}$ logits, against a median selector magnitude of
$3.9$--$8.5$ logits on the same inputs: single-precision round-off, six orders
of magnitude below the signal.

No estimation error survives, and the reason is structural rather than
fortunate. The patch and the weight edit apply the \emph{same} estimated
projector, one to the carrier's activation and the other to its output weight
matrix; since the edit lies downstream of the activation it acts on, the two
produce the same write-in exactly, and whatever error is made in estimating
the subspace is made identically on both sides and cancels.
That this test could have failed is worth recording: re-estimating the
subspace on different data for the patch route \emph{alone} (the
configuration an estimation-error account describes) raises the same gap to
a median of up to $0.69$ and a maximum of $6.6$ logits. An estimation error of
the size these subspaces actually carry is thus easily large enough to see;
the protocol simply does not expose it.

The residual between C1a and $1$ therefore has a different and simpler source.
Across the $2000$ probe measurements there is not one input on which the
ablated network returned one of the two candidate readings and the idealized
model predicted the other; the entire residual is the $30$ inputs ($1.5\%$,
nearly all on seed 22) whose argmax fell on a \emph{third} token, outside the
binary the prediction ranges over. On a single-carrier configuration C1a is
thus exactly one minus the rate of such off-binary answers: a fact about
where the ablated network's argmax lands, not a measure of disagreement with
the idealized model. This reinforces rather than weakens the reading above:
the D1 points of Figure~\ref{fig:interaction} calibrate the measurement
protocol, and it is the joint configurations that carry the evidential weight.

\begin{figure}[htbp]
\centering
\includegraphics[width=.85\linewidth]{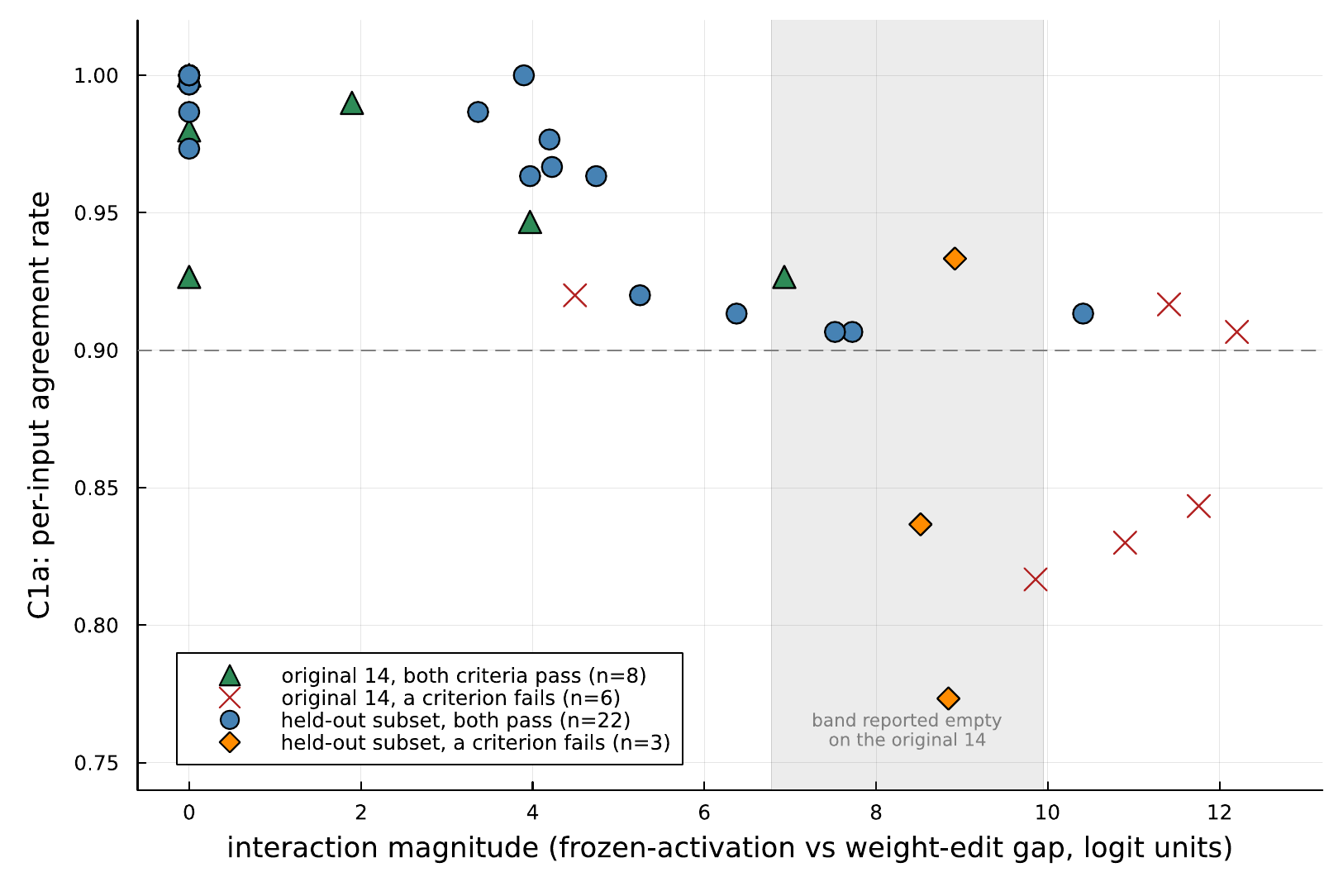}
\caption{Agreement rate (C1a) against measured interaction magnitude over all
$39$ (instance, subset) configurations, on probes held out from every other
number in this paper. Triangles and crosses are the $14$ configurations
originally reported (D1/DJ/DJA per instance); circles and diamonds are the
$25$ held-out subsets. The shaded region is the band that was empty on the
original $14$ and is used here only to show that it is \emph{not} empty on the
full set: $7$ configurations fall inside it. Passing configurations reach
interaction $10.41$ and failing ones fall to $4.49$, so the two groups
overlap. What survives is monotone rather than sharp: Spearman
$\rho=-0.83$ over all $39$ ($-0.85$ on the $25$ held-out ones alone), and the
pass rate over the configurations with nonzero interaction falls
$7/8\to6/8\to1/7$ across terciles of interaction magnitude. The five
single-carrier ablations sit at interaction $\approx0$ (at most
$8\times10^{-6}$ logits, see the text) for protocol reasons rather than
theoretical ones, and are not evidence either way.}
\label{fig:interaction}
\end{figure}

\paragraph{Does the separation hold out of sample?} The band above was read
off the $14$ configurations that produced it, which makes it descriptive, not
a validated threshold. Because the five checkpoints admit many more ablation
subsets than the three per instance reported, we could test it. We froze the
published band ($6.78$, $9.95$) and the decision rule (a configuration
passing both criteria must fall below the band, one failing either must fall
above it), and applied them unchanged to \emph{every} non-empty subset of
each instance's carrier set: $39$ configurations, of which $25$ had never been
used to fix anything, on a third disjoint set of probe seeds.

The rule fails. Nine of the $39$ configurations violate it, six of them
out-of-sample; seven land inside the supposedly empty band. Three of the
original $14$ change which side of the band they fall on when merely
re-measured on held-out probes, so the separation is not stable even in
sample. A split-half control on the same probes attributes part of this to
sampling noise ($3$ of $39$ configurations flip their pass/fail verdict
between the two halves, all of them sitting near the $0.90$ threshold), but
not the overlap itself, which is far too large for that.

The monotone relationship, on the other hand, is robust. Interaction magnitude
and C1a have Spearman $\rho=-0.83$ over all $39$ configurations (two-sided
permutation test, $p<10^{-5}$), $-0.85$ over the $25$ held-out ones alone
($p<10^{-5}$), and $-0.81$ over the $14$ that are both held-out and have
nonzero interaction ($p\approx7\times10^{-4}$), so the association is not an
artifact of the single-carrier configurations sitting at zero: if anything
it is slightly tighter on the held-out subsets than on the full set, the
opposite of what overfitting the original band would predict. We therefore
withdraw the claim of a clean separating threshold and keep the weaker one the
data support: the interaction magnitude (the same idealization gap
Theorem~\ref{thm:interaction} gives in closed form for a single head and its
own layer's MLP, measured here for every tested subset including the ones the
theorem does not cover analytically, such as carriers spanning more than one
layer) predicts, monotonically and with a large effect, how far the
idealized model of Sections~\ref{sec:model}--\ref{sec:dissociation} is from
the network whose weights were actually edited. That is still the qualitative content the theory
asserts (whether the idealization holds is governed by the size of a
quantity derived from the architecture alone, before any data-dependent
fitting), but it does not license reading a numerical cutoff off these
experiments.

Two things are worth separating from that conclusion. The \emph{measurement
gap} between the frozen-activation protocol and the true edit ($\le
8\times10^{-6}$ logits on all five single-carrier configurations, by the
patch-equals-weight-edit identity of Proposition~\ref{prop:patchedit} rather
than by Corollary~\ref{cor:exactness} whenever the ablated carrier is a head)
is not the same quantity as $\Delta(x)$ itself, which the corollary
guarantees absent only when the ablated carrier is the MLP. And the
thresholds themselves (C1a $\ge0.90$, C1b within $\pm0.10$) were fixed in the
verification script before any checkpoint was probed, so the failure above is
a failure of the band, not of a threshold chosen after the fact.

\subsection{A second task and architecture}
\label{sec:second-task}

Everything above illustrates the theory on one task and one family of
architectures. To check that the predictions are not an artifact of the
marker task's specific mechanism (an externally fixed derangement applied
to a displayed token, then read in a single, fixed direction), we built a
second, freshly trained instance with a different conditional mechanism and a
different architecture, and re-ran the two checks that do not themselves
presuppose a matched-pair, two-branch structure, plus a search for a
polarity-reversal analogue.

\paragraph{Task and architecture.} The context is $4$ key--value pairs from a
shared vocabulary of $12$ symbols (keys mutually distinct, values mutually
distinct, drawn independently), mixed in random order, followed by a marker
and a displayed token. Under $m_{\mathrm F}$ the displayed token is a pair's
key $k_t$ and the answer is its value $v_t$: ordinary associative recall.
Under $m_{\mathrm R}$ the displayed token is $v_t$ itself and the answer is
$k_t$: a genuine \emph{inverse} lookup, from value back to key, that
$m_{\mathrm F}$ never requires and that involves no externally fixed
permutation of any kind. The shared contrast axis is
$s_{\mathrm{op}}(x):=\mathrm{logit}(v_t)-\mathrm{logit}(k_t)$: format
$\mathrm F$ is correct when this favors $v_t$ (positive), format $\mathrm R$
when it favors $k_t$ (negative), structurally the same matched-pair
mechanism as $s_{\mathrm{op}}=\mathrm{logit}(\mathrm{lit})-\mathrm{logit}
(\mathrm{inv})$ in Section~\ref{sec:experiments}, but with the two roles
filled by a genuine forward/inverse recall rather than by a token relabeled
under a fixed $\sigma$. The network is a $3$-layer, $3$-head Llama-style
transformer, width $48$ (head dimension $16$), SwiGLU hidden width $96$,
smaller and shallower than the $4$-layer, $4$-head, width-$64$ networks
above, on a $14$-token vocabulary rather than $10$. It was trained fresh, once,
for this check (never on the checkpoints used elsewhere in this paper), with
the same AdamW/warmup recipe as Appendix~\ref{app:repro}, and reached
$99.5\%$/$98.5\%$ accuracy on formats $\mathrm F$/$\mathrm R$ after $2750$ of
a $6000$-step budget: the same qualitative training dynamic
(a long plateau below $35\%$ followed by a rapid transition) as the marker
task's own instances.

\paragraph{Carriers.} The same patch-recovery sweep and $r\ge0.25$ threshold
used throughout this paper finds only two carriers on this checkpoint, both
MLPs (layer $1$, $r=0.88$; layer $2$, $r=0.35$), giving $3$ non-empty subsets:
too few for the interaction/fidelity relationship below to mean anything
statistically. The full ranked sweep shows why: after those two, recovery
drops to $0.19$ (layer $3$'s MLP) and $0.18$ (a layer-$1$ head), then drops
again, by more than a factor of two, to $0.09$ and below for every remaining
site. We report both readings rather than picking one: at the paper's own
threshold, this task concentrates onto two MLPs and nothing else passes; at
$r\ge0.15$ (still twice the next site down, not a threshold tuned to reach
a target count) four carriers survive, spanning all three layers and
including one head, giving $15$ non-empty subsets. The three checks below use
this second, wider carrier set, since it is the one that can actually support
them; where the two readings disagree, both are reported.

\paragraph{Patch-equals-weight-edit exactness.} On the single strongest
carrier (layer $1$'s MLP, $r=0.88$) and $120$ fresh probes, the gap between
the frozen-activation patch and the genuinely weight-edited network has
median $9.5\times10^{-7}$ and maximum $3.8\times10^{-6}$ logits, against a
median selector scale of $6.99$, the same six-orders-of-magnitude,
single-precision-floor gap as the five D1 configurations of
Section~\ref{sec:experiments} (Proposition~\ref{prop:patchedit} makes no
reference to the task), and the falsifiability control (the same protocol
with the patch route's subspace re-estimated on disjoint data) again shows
what a genuine mismatch would look like: median $2.44$, maximum $11.6$
logits, three orders of magnitude larger.

\paragraph{Interaction versus fidelity.} Over the $15$ configurations at
$r\ge0.15$, interaction magnitude and the agreement rate C1a are
anti-correlated, Spearman $\rho=-0.90$ (two-sided permutation test,
$p\approx3\times10^{-5}$), the same relationship, on a different task and
a smaller, differently shaped network, that Section~\ref{sec:experiments}
reports at $\rho=-0.83$ over its own $39$ configurations. (At $r\ge0.25$
the same relationship is present on the $3$ available configurations,
$\rho=-1.0$, but three points do not constitute independent evidence of a
monotone relationship on their own; we report it for completeness, not as a
replication.) The off-binary (``other'') rate is higher on this task, up to
$47.5\%$ on the largest and one other joint configuration against $1.5\%$ on
the marker task's D1 points, which is consistent with a smaller, shallower network
producing more answers outside the two candidate readings once several
carriers are ablated together, and is exactly the kind of number
Section~\ref{sec:experiments} already argues C1a's residual should track.

\paragraph{A polarity-reversal analogue.} Scanning every nested pair of the
$15$ configurations against the $60$ probed pairs finds eight instances,
across two distinct probe pairs, in which enlarging an ablated subset that
already collapses a pair cleanly reverses which branch it collapses onto:
the same phenomenon Figure~\ref{fig:seed44} illustrates on the marker task.
The cleanest is layer-internal, exactly as on seed 44: ablating layer $1$'s
MLP alone collapses one probed pair onto one branch, while ablating that same
MLP \emph{together with} layer $1$'s own head collapses it onto the other.
A second instance spans more than one layer, adding layer $2$'s MLP to an
already-collapsing layer-$1$-and-layer-$3$ pair to reverse the branch; it is
not covered by this paper's own two-carrier, single-block composition
(Theorem~\ref{thm:interaction}), and extending the analysis to carriers at
different layers is left to future work. By Corollary~\ref{cor:invariance},
this is again evidence that at least one of each pair of subsets fails the
exact collapse hypotheses, consistent with the idealized model rather than
contradicting it.

\paragraph{What this does and does not establish.} All three checks land the
same way they do on the marker task, which is the point of running them: the
predictions are not an artifact of one conditional mechanism or one
architecture shape. It remains a single trained instance, not five, and the
carrier count needed widening past this paper's own stated threshold to give
the second check enough points to be informative: both are limitations we
would rather state than round past. We did not attempt a second
out-of-sample band test or a second robust-collapse-availability measurement
here: both are properties of \emph{how much data} a claim was checked
against, which a single additional instance cannot settle either way, and
re-running them here would not add evidence beyond what Section~\ref{sec:experiments}
already reports on that question.

\section{Discussion}
\label{sec:discussion}

\paragraph{What is established, and how far it reaches.} The flip criterion
(Lemma~\ref{lem:flip}), the collapse theorem (Theorem~\ref{thm:collapse}) and
the dissociation theorem (Theorem~\ref{thm:dissociation}) rest on
Assumption~\ref{ass:readout} and on the additive decomposition of
Section~\ref{sec:model}, and on nothing else. In particular none of them uses
any property of transformers: they hold for any residual computation with a
linear readout, and their content is arithmetic on the two numbers
$\beta_i\delta_i$ and $\beta_i\alpha_i(x_B)$ that the two interventions
respectively move. Proposition~\ref{prop:robust} removes the objection that
the collapse criterion is stated as an exact equality and therefore never
applies: the error is an identity, not an approximation, and the criterion has
a tolerance form. What that form buys on real weights is measured, and
reported honestly as partial, in Section~\ref{sec:experiments}. What they do
\emph{not} establish is that a real network
satisfies the decomposition; that is the role of
Theorem~\ref{thm:interaction}, which computes the error exactly for the
composition in which the independence of carriers is architecturally false,
and of
Corollary~\ref{cor:exactness}, which shows the failure is identically absent
whenever the ablated carrier is the MLP, and, by
Proposition~\ref{prop:readout}, exactly bounded whenever it is not, a bound
whose governing constant is in fact computable in closed form from the
trained weights alone, which a companion analysis to this one exhibits
explicitly. The
scope of the idealized model is therefore not assumed but delimited from
inside it.

What the experiments establish is weaker than what a first pass suggested, and
we prefer to say so in both places. The interaction magnitude tracks the
idealized model's accuracy monotonically and strongly across $39$ ablation
configurations, but the clean numerical separation visible on the first $14$
is not reproducible on the other $25$. A reader should take from
Section~\ref{sec:experiments} that the theory identifies the right
\emph{quantity}, not that it supplies a usable decision threshold; supplying
one would require establishing that the cutoff transfers across networks,
which these five instances do not show.

\paragraph{What remains open.} The theory answers \emph{when} collapse and
dissociation occur given a specific ablated subset, but not the converse
question of \emph{which} subset to choose in order to realize a target
outcome: Corollary~\ref{cor:invariance} shows that any two idealized-model-satisfying
subsets must agree, but gives no procedure for constructing one, and does not
address whether a subset can be \emph{designed}, from measurable quantities
alone, to steer the collapse polarity at will. Two further gaps are
specific to Theorem~\ref{thm:interaction}: the analysis is not extended past the block in which
the two carriers reside, though Remark~\ref{rem:propagation} identifies the
two mechanisms (normalization-driven attenuation and linear-map
amplification) through which a quantitative account would proceed;
and multi-layer carrier interactions are not covered by the two-carrier
composition considered here. Each is a well-posed mathematical question
rather than an open-ended empirical one, and each is a natural target for the
idealized model of Section~\ref{sec:model} to be extended rather than
replaced. A fourth gap is empirical rather than mathematical: every result of
Section~\ref{sec:experiments} is illustrated on a small transformer trained
on a synthetic conditional task, chosen because it makes the two branches of
the conditional and the ground-truth answer unambiguous by construction.
Mechanistic interpretability is ultimately concerned with models trained on
natural language, and nothing in Sections~\ref{sec:model}--\ref{sec:interaction}
is specific to the synthetic task; extending the theory past a single
residual block and testing it against an emergent circuit in a real
pretrained model is exactly the direction a companion analysis takes, and we
leave that extension, and its own genuinely mixed result, to it rather than
duplicate it here. The out-of-sample failure reported in
Section~\ref{sec:experiments} nonetheless sharpens what any such extension
would have to show. It is not enough to exhibit a correlation between
interaction magnitude and idealization error; that already replicates on
this paper's own synthetic instances. The open question is whether any
threshold on that magnitude transfers, between configurations or between
networks, and our five instances answer it negatively at the only scale we
tested here, which makes the question more interesting rather than less.

\section{Conclusion}
\label{sec:conclusion}

Activation patching and weight-space ablation are not two measurements of one
causal quantity. They are two operators on two different objects, and the
useful question is not whether they agree but exactly when, and by how much
they fail to when they do not. Within one idealized model this becomes a
matter of arithmetic: patching moves a contrast, ablation removes a level,
and the collapse theorem of Section~\ref{sec:collapse} says precisely which
subsets of carriers force a conditional onto a single branch and which branch
survives.

The practical reading is that a negative ablation result is not, on its own,
evidence that a component is unimportant, and a positive patching result is
not, on its own, evidence that it is necessary, not as a caution about
noisy measurement, but because the two experiments interrogate different
quantities that a redundant code routinely separates. Where the idealization
itself fails, the failure is not noise either: it is an interaction term of
known order, with a provable zero whenever the ablated carrier is an MLP and
an exact second-order bound otherwise whose constant is computable from the
trained weights, and it is measurable on the network one is actually studying.

Measurable is not the same as calibrated, and we close on that distinction
rather than blur it. Across thirty-nine ablation configurations on five
trained networks the interaction magnitude orders the idealization's
accuracy strongly and reproducibly, but the threshold that appeared to
separate success from failure on the first fourteen did not survive the
other twenty-five: the quantity singled out by the theory is the right one,
how to read a numerical cutoff off it is not yet settled, and we would rather
leave that explicit than let a tidy figure imply otherwise.

\appendix

\section{Experimental Details and Reproducibility}
\label{app:repro}

Section~\ref{sec:experiments} is deliberately compressed; this appendix gives
what is needed to reproduce it.

\paragraph{Task.} Key and value tokens are drawn from a vocabulary of
$V=8$ symbols, with two additional marker tokens $m_A,m_B$, so the model's
vocabulary has $10$ entries. Each sequence presents $3$ key--value pairs in
random order ($6$ tokens) followed by a marker and a displayed token, giving
sequence length $8$; the target is the value of the true key, read at the
final position. The derangement $\sigma$ is drawn once from a fixed
pseudorandom seed ($42$) and is \emph{identical across all instances}, so only
the trained model varies between seeds.

\paragraph{Architecture and training.} Each instance is a $4$-layer decoder
transformer in the Llama style: model width $d=64$, $4$ attention heads
(head dimension $16$), SwiGLU MLP with hidden width $128$, pre-normalization
by RMSNorm with $\varepsilon=10^{-6}$, learned positional embeddings, weights
in float32. There is no final normalization: the unembedding is applied
directly to the residual stream, so the logit contrast is an exactly affine
functional of $F(x)$ and Assumption~\ref{ass:readout} holds exactly rather
than approximately for these networks. Training is $5000$ optimizer steps at batch $64$ with AdamW
($\beta_1=0.9$, $\beta_2=0.999$, $\varepsilon=10^{-8}$), peak learning rate
$10^{-3}$, linear warmup over $250$ steps then cosine decay to
$0.1\times$ peak. Only instances reaching accuracy $\ge0.95$ on \emph{both}
formats are analyzed; this gate is applied before any ablation. Four
instances use initialization seed $S\in\{11,22,33,44\}$ with training-data
seed $1000S+123$; the fifth (``inst.\ 2'') was trained earlier under the same
configuration with a different seed. No model was retrained for any
measurement reported here: all of them read the same five checkpoints,
with one exception: the second task and architecture of
Section~\ref{sec:second-task} is a genuinely separate, freshly trained
instance (\texttt{notebook/\allowbreak bidir\_recall\_task\_experiment.jl},
initialization seed $1$, training seed $123$, $2750$ steps at batch $64$ before the
$\ge0.97$/$\ge0.97$ early-stop gate, otherwise the same optimizer recipe),
never used for any other number in this paper.

\paragraph{Carrier selection.} The candidate set is every attention head and
every MLP output, $4\times(4+1)=20$ sites. On $8$ matched clean pairs (seed
$4242$) each site is patched from donor to receiver in turn and scored by the
normalized recovery of the literal-versus-inverted logit contrast; $r$ denotes
the mean over pairs. \textbf{A site is a carrier when $r\ge0.25$.} The
\emph{carrier layer} is the shallowest layer containing a carrier. The three
configurations of Section~\ref{sec:experiments} are D1 (the highest-$r$
carrier in the carrier layer, alone), DJ (all carriers in the carrier layer),
and DJA (all carriers); DJA is omitted when it coincides with DJ, which is why
seed 11 contributes two configurations rather than three.

\paragraph{Ablation subspaces.} For each carrier, donor--receiver activation
differences are collected at the last two sequence positions over $32$ matched
pairs (seed $1618$), giving $64$ vectors. Their SVD determines
$U$ as the leading left singular vectors, taking \textbf{the smallest $k$
reaching $90\%$ of the squared-singular-value energy, capped at $k\le4$ for a
head and $k\le8$ for an MLP.} The weight edit is
$W_O[:,\mathcal H]\leftarrow W_O[:,\mathcal H](I-UU^\top)$ for a head with
column block $\mathcal H$, and $W_2\leftarrow(I-UU^\top)W_2$ for an MLP. These
are the concrete instances of \eqref{eq:ablation} referred to in
Remark~\ref{rem:bridge}, with the directions estimated from data.

\paragraph{Evaluation criteria.} C1a and C1b are as defined in
Section~\ref{sec:experiments}; the thresholds $0.90$ and $\pm0.10$ were fixed
in the verification script before any checkpoint was probed and were not
revised. Rates are computed on filtered distributions on which the systematic
wrong answer is a well-defined token distinct from the correct one.

\paragraph{Probe seeds.} Every measurement added after the original run uses
probe seeds disjoint from it and from each other, so that no two reported
numbers share an evaluation sample:
\begin{itemize}
\item the original configuration run uses $90210/31337$;
\item the single-carrier protocol-gap measurement uses $20260726/20260727$
($400$ probe inputs per instance);
\item the out-of-sample band test uses $20260728/20260729$
($300$ per configuration);
\item the robust-hypothesis measurement uses $20260801/20260802$
($120$ matched pairs per configuration);
\item the second task and architecture of Section~\ref{sec:second-task} uses
its own carrier-sweep seed $53421$, delta seed $91827$, and probe seeds
$20260810/20260811$ ($60$ matched pairs), all independent of every seed above
and of each other, since it is a separate checkpoint with no other number in
this paper to stay disjoint from.
\end{itemize}
The permutation tests use $2\times10^5$ resamples.

\paragraph{Code and data availability.} All experiments run against the
research repository accompanying this submission; no external data is used and
no model is downloaded. Each script writes a JSON results file that the
figure scripts read directly, so no number in a figure is transcribed by
hand. The relevant files are:
\begin{itemize}
\item \texttt{notebook/marker\_task\_experiment.jl} --- task, architecture,
training.
\item \texttt{notebook/marker\_seed\_matrix.jl} --- per-seed training and
carrier sweep.
\item \texttt{notebook/marker\_conj1\_verify.jl} --- the original $14$
configurations, C1a/C1b and interaction magnitude.
\item \texttt{notebook/verify\_marker\_interaction\_theorem.jl} --- numerical
check of Theorem~\ref{thm:interaction}.
\item \texttt{notebook/verify\_d1\_protocol\_gap.jl} ---
Proposition~\ref{prop:patchedit}, the single-carrier gap.
\item \texttt{notebook/verify\_config\_band\_oos.jl} --- the
$39$-configuration out-of-sample test.
\item \texttt{notebook/verify\_robust\_collapse.jl} ---
Proposition~\ref{prop:robust}, the $\bar q_S$/$\Psi_S$/$\bar s$ measurement.
\item \texttt{notebook/verify\_interaction\_spearman.jl} --- the Spearman
correlations and permutation $p$-values reported in Section~\ref{sec:experiments},
from the JSON output of the two preceding scripts, computed with tie-corrected
rank correlation matching, respectively:
\begin{itemize}
\item \texttt{StatsBase.corspearman} (Julia);
\item \texttt{scipy.stats.spearmanr} (Python);
\item R's \texttt{cor(method="spearman")}.
\end{itemize}
An order-dependent,
non-tie-corrected rank statistic reproduces the values reported in an earlier
draft of this section and is retained in the script as a documented control.
\item \texttt{notebook/bidir\_recall\_task\_experiment.jl} --- the second
task's architecture, training, and checkpoint of Section~\ref{sec:second-task}.
\item \texttt{notebook/verify\_bidir\_replication.jl} --- its three checks,
run at both carrier thresholds discussed there.
\end{itemize}
The five marker-task checkpoints are stored under
\texttt{notebook/marker\_ckpt/}, the second task's checkpoint under
\texttt{notebook/bidir\_ckpt/}.

\end{document}